\documentclass{article}
\ifdefined\usebigfont

\usepackage{times}
\usepackage[fontsize=13pt]{scrextend}
\usepackage[left=1.56in,right=1.56in,top=1.74in,bottom=1.74in]{geometry}
\usepackage[numbers]{natbib}
\else
\usepackage[final]{neurips_2018}
\fi
\usepackage[utf8]{inputenc} % allow utf-8 input
\usepackage[T1]{fontenc}    % use 8-bit T1 fonts

\usepackage{hyperref,url} 
\usepackage{booktabs,paralist}
\usepackage{amsfonts,amsmath,amssymb,amsthm} 
\usepackage{nicefrac,microtype,physics} 
\usepackage{color}
\usepackage{graphicx,subcaption}
\usepackage{thmtools,thm-restate}

\DeclareMathOperator{\diag}{\text{diag}}

\newcommand{\remove}[1]{{}}

\renewcommand{\norm}[1]{\lVert #1 \rVert}
\renewcommand{\c}{\mathcal}
\renewcommand{\b}{\mathbf}

\newcommand{\bR}{\mathbb{R}}

\renewcommand{\hat}{\widehat}
\renewcommand{\tilde}{\widetilde}
\renewcommand{\bar}{\overline}

\newcommand{\innerprod}[2]{\langle #1,#2 \rangle}
\newcommand{\mybox}{\hfill\(\Box\)}

\DeclareMathOperator*{\argmin}{\mathrm{argmin}}

\newcommand{\st}{\text{ s.t. }}
\newcommand{\Th}{\text{th}}
\newcommand{\ie}{\textit{i.e.,} }
\newcommand{\eg}{\textit{e.g.,} }
\newcommand{\tand}{\text{ and }}

\renewcommand{\u}[1][]{\ifthenelse{\equal{#1}{}}{\mathbf{w}}{\mathbf{w}^{(#1)}}}
\newcommand{\w}[1][]{\ifthenelse{\equal{#1}{}}{\boldsymbol{\beta}}{\boldsymbol{\beta}^{(#1)}}}
\newcommand{\z}[1][]{\ifthenelse{\equal{#1}{}}{\mathbf{z}}{\mathbf{z}^{(#1)}}}
\newcommand{\h}[1][]{\ifthenelse{\equal{#1}{}}{{\mathbf{h}}}{{\mathbf{h}^{(#1)}}}}
\newcommand{\bdelta}[1][]{\ifthenelse{\equal{#1}{}}{{\boldsymbol{\delta}}}{\boldsymbol{\delta}^{(#1)}}}
\newcommand{\brho}[1][]{\ifthenelse{\equal{#1}{}}{{\boldsymbol{\rho}}}{\boldsymbol{\rho}^{(#1)}}}
\newcommand{\x}{\mathbf{x}}

\renewcommand{\P}{\mathcal{P}}
\newcommand{\ci}{\mathbf{\mathrm{{i}}}}
\newcommand{\e}{\mathrm{e}}

\newtheorem{claim}{Claim}
\newtheorem*{assumption*}{Assumptions}
\newtheorem*{definition*}{Definition}
\newtheorem{theorem}{Theorem}
\newtheorem{lemma}[theorem]{Lemma}

\title{Implicit Bias of Gradient Descent on Linear Convolutional Networks}
{{\centering\author{Suriya Gunasekar \\ TTI at Chicago, USA \\ \texttt{suriya@ttic.edu}   \And Jason D. Lee  \\ USC Los Angeles, USA\\ \texttt{jasonlee@marshall.usc.edu} \And Daniel Soudry \\ Technion, Israel \\ \texttt{daniel.soudry@gmail.com}  \And Nathan Srebro \\ TTI at Chicago, USA \\ \texttt{nati@ttic.edu}}
}
\begin{document}
\maketitle
\begin{abstract}
%We characterize the implicit bias of  gradient descent in learning  multi-layer linear convolutional and fully connected networks for binary classification task. 
We show that gradient descent on full width linear convolutional networks of depth $L$ converges to a linear predictor related to the $\ell_{2/L}$ bridge penalty in the frequency domain.  This is in contrast to  fully connected linear networks, where regardless of depth, gradient descent converges to the $\ell_2$ maximum margin solution. %hard margin support vector machine solution. 
\end{abstract}
\section{Introduction}
Implicit biases introduced by optimization algorithms  play an crucial role in learning deep neural networks \citep{neyshabur2015search,neyshabur2015path,hochreiter1997flat, keskar2016large, chaudhari2016entropy, dinh2017sharp,andrychowicz2016learning, neyshabur2017geometry,zhang2017understanding,wilson2017marginal,hoffer2017train,Smith2018}. Large scale neural networks used in practice are highly over-parameterized with far more trainable model parameters compared to the number of training examples. Consequently, optimization objectives for learning such high capacity models have many global minima that fit training data perfectly.
However,  minimizing the training loss using specific optimization algorithms take us to not just any global minima, but some special global minima, \eg global minima minimizing some regularizer $\c{R}(\w)$. In over-parameterized models, specially deep neural networks, much, if not most, of the inductive bias of the learned model comes from this implicit regularization from the optimization algorithm. Understanding the implicit bias, \eg via characterizing $\c{R}(\w)$, is thus essential for understanding how and what the model learns.

For example, in linear regression we understand how minimizing an under-determined model (with more parameters than samples) using gradient descent yields the minimum $\ell_2$ norm solution, and for linear logistic regression trained on linearly separable data, \citet{soudry2017implicit} recently showed that gradient descent converges in the direction of the hard margin support vector machine solution,  even though the norm or margin is not explicitly specified in the optimization problem.  Such minimum norm or maximum margin solutions are of course very special among all solutions or separators that fit the training data,  and in particular can ensure generalization \cite{bartlett2003rademacher,kakade2009complexity}.  

Changing the optimization algorithm, even without changing the model, changes this implicit bias, and consequently also changes generalization properties of the learned models \citep{neyshabur2015path,keskar2016large,wilson2017marginal,gunasekar2017implicit,gunasekar2018characterizing}. For example, for linear logistic regression,  using coordinate descent instead of gradient descent return a maximum $\ell_1$ margin solution instead of the hard margin support vector solution solution---an entirely different inductive bias \cite{telgarsky2013margins,gunasekar2018characterizing}. 

Similarly, and as we shall see in this paper, changing to a different parameterization of the same model class can also dramatically change the implicit bias \cite{gunasekar2017implicit}.  In particular,  we study the implicit bias of optimizing multi-layer   fully connected linear networks, and linear convolutional networks (multiple full width convolutional layers followed by a single fully connected layer) using gradient descent. 
Both of these types of models ultimately implement linear transformations, and can implement any linear transformation.  The model class defined by these networks is thus simply the class of all linear predictors, and these models can be seen as mere (over) parameterizations of the class of linear predictors.  Minimizing the training loss on these models is therefore entirely equivalent to minimizing the training loss for linear classification.  Nevertheless, as we shall see, optimizing these networks with gradient descent leads to very different solutions. 

In particular, we show that for fully connected networks with single output, optimizing the exponential loss over linearly separable data using gradient loss  again converges to the homogeneous hard margin support vector machine solution. This holds regardless of the depth of the  network, and hence, at least with a single output, gradient descent on fully connected networks has the same implicit bias as direct gradient descent on the parameters of the linear predictor.  In contrast, training a linear convolutional network with gradient descent biases us toward linear separators that are sparse in the frequency domain.  Furthermore, this bias changes with the depth of the network, and a network of depth $L$ (with $L-1$ convolutional layers), implicitly biases towards  minimizing the $\|{\hat{\w}}\|_{\nicefrac{2}{L}}$  bridge penalty with $\nicefrac{2}{L}\le 1$ of the Fourier transform $\hat{\w}$ of the learned linear predictor $\w$ subject to margin constraints (the gradient descent predictor reaches a stationary point of the $\|{\hat{\w}}\|_{\nicefrac{2}{L}}$ minimization problem).  This is a sparsity inducing regularizer, which induces sparsity more aggressively as  the depth increases.

Finally, in this paper we focus on characterizing \textit{which} global minimum does gradient descent on over-parameterized linear models converge to, while assuming that for appropriate choice of step sizes gradient descent iterates asymptotically minimize the optimization objective.  A related challenge in neural networks,  not addressed in this paper, is an answer to \textit{when} does gradient descent minimize the non-convex empirical loss objective  to reach \textit{a} global minimum. This problem while hard in worst case, has been studied for linear networks. Recent work have concluded that with  sufficient over-parameterization (as is the case with our settings),  loss landscape of linear models are well behaved and all local minima are global minima making the problem tractable  \cite{burer2003nonlinear,journee2010low,kawaguchi2016deep,nguyen2017loss,lee2016gradient}.

\paragraph{Notation} 
We typeface vectors with bold characters \eg $\u,\w,\x$. Individual entries of a vector $\z\in\bR^D$ are indexed using  $0$ based indexing as $\z{[d]}$ for $d=0,1,\ldots,D-1$. 
Complex numbers are represented  in the polar form as $z=|z|\e^{\ci\phi_z}$ with $|z|\in\bR_+$ denoting the magnitude of $z$ and $\phi_z\in[0,2\pi)$ denoting the phase. $z^*=|z|\e^{-\ci\phi_z}$ denotes the complex conjugate of $z$. The complex inner product between $\z,\w\in\mathbb{C}^{D}$ is given by $\innerprod{\z}{\w}=\sum_{d=1}^D \z{[d]} \w^*{[d]}=\z^\top{\w^*}$. The $D^\Th$ complex root of $1$ is denoted by $\omega_D=\e^{-\frac{2\pi\ci}{D}}$. For  $\z\in\bR^D$ we use the notation $\hat{\z}\in\mathbb{C}^D$ to denote  the representation of $\z$ in the discrete Fourier basis given by, 
$\hat{\z}[d]=\frac{1}{\sqrt{D}}\sum_{p=0}^{D-1} \z{[p]}\omega^{pd}_D.$ For integers $D$ and $a$, we denote the modulo operator as $a\text{ mod }D=a-D\left\lfloor\frac{a}{D}\right\rfloor$. 
Finally, for multi-layer linear networks (formally defined in Section~\ref{sec:lnn}), we will use $\u\in\c{W}$ to denote parameters of the model in general domain $\c{W}$, and $\w_{\u}$ or simply $\w$ to denote the equivalent linear predictor.

\section{Multi-layer Linear  Networks}\label{sec:lnn}
We consider feed forward linear networks that map input features $\x\in\bR^D$ to a single real valued output $f_{\u}(\x)\in\bR$, where $\u$ denote the parameters of the network. Such networks can be thought of as  directed acyclic graphs where each edge is associated with a weight, and the value at each node/unit is the weighted sum of  values from the parent nodes. The input features form source nodes with no incoming edges and the output is a sink node with no outgoing edge. Every such network realizes a linear function $\x\to\innerprod{\x}{\w_{\u}}$, where $\w_{\u}\in\bR^D$ denotes the effective linear predictor. %We consider two special cases of  linear networks, a fully connected network and convolutional network.

In multi-layer networks, the nodes are arranged in layers, so  an $L$--layer network represents a composition of $L$ linear maps. We use the convention that, the input  $\x\in\bR^D$ is indexed as the zeroth layer $l=0$, while the output forms the final layer with $l=L$. The outputs of  nodes in layer $l$ are denoted by $\h_l\in\bR^{D_l}$,  where $D_l$ is the number of nodes in layer $l$. We also use $\u_l$ to denote the parameters of the linear map between $\h_{l-1}$ and $\h_l$, and $\u=[\u_l]_{l=1}^L$ to denote the collective set of all parameters of the linear network. 

\textbf{Linear fully connected network}\quad 
In a  fully connected linear network, the nodes between successive layers $l-1$ and $l$ are densely connected with edge weights $\u_l\in\bR^{D_{l-1}\times D_l}$, and all the weights are independent parameters. This model class is parameterized by $\u=[\u_l]_{l=1}^L\in\prod_{l=1}^L \bR^{D_{l-1}\times D_l}$ and the computation for intermediate nodes $\h_l$ and the composite linear map $f_{\u}(\x)$ is given by,
\begin{equation}
\h_l=\u_l^\top\h_{l-1}\quad \tand \quad f_{\u}(\x)=\h_L=\u_L^\top\u_{L-1}^\top\ldots\u_1^\top\x.
\label{eq:fcn}
\end{equation}

\textbf{Linear convolutional network}\quad We consider one-dimensional convolutional network architectures where each non-output layer has exactly $D$ units (same as the input dimensionality) and the linear transformations from layer ${l-1}$ to layer $l$ are given by the following circular convolutional operation\footnote{We follow the convention used in neural networks literature that refer to the operation in \eqref{eq:conv} as convolution, while in the signal processing terminology, \eqref{eq:conv} is known as the discrete  circular cross-correlation operator}  parameterized by full width  filters with weights $[\u_l\in\bR^D]_{l=1}^{L-1}$.  For $l=1,2,\ldots,L-1$, 
\begin{equation}
\quad\h_l[d]=\frac{1}{\sqrt{D}}\sum_{k=0}^{D-1}\u_l[k]\,\h_{l-1}\left[(d+k)\text{ mod } D\right]:=\left(\h_{l-1}\star \u_l\right)[d].
\label{eq:conv}
\end{equation}
%where we use the notation $\b{a}\star\b{b}$ to denote the circular convolutional operator $\star:\bR^d\times\bR^D\to\bR^D$. 

The output layer is fully connected and parameterized by weights $\u_L \in \bR^D$. The  parameters of the model class therefor consists of $L$ vectors of size $D$ collectively denoted by ${\u=[\u_l]_{l=1}^{L}\in\prod_{l=1}^L \bR^{D}}$, and the composite linear map $f_{\u}(\x)$  is given by:
\begin{equation}
f_{\u}(\x)=\left((((\x\star\u_1)\star\u_2)\ldots)\star\u_{L-1}\right)^\top \u_L.
\label{eq:cn}
\end{equation}
\textit{Remark: }We use circular convolution with a scaling of $\nicefrac{1}{\sqrt{D}}$ to make the analysis cleaner. For convolutions with zero-padding, we expect a similar behavior. Secondly, since our goal here to study implicit bias in sufficiently over-parameterized models, we only study full dimensional convolutional filters. In practice it is common to have filters of width $K$  smaller than the number of input features, which can change the implicit  bias.

The  fully connected and convolutional linear networks described above can both be represented in terms of a mapping $\P:\c{W}\to\bR^D$ that maps the input parameters $\u\in\c{W}$ to a linear predictor in $\bR^D$, such that the output of the network is given by $f_{\u}({\x})=\innerprod{\x}{\P(\u)}$.  
For fully connected networks, the mapping  is given by $\P_{full}(\u)=\u_1\u_{2}\ldots\u_L$, and for convolutional networks, $\P_{conv}(\u)=\left(\big((\u_L^{\downarrow}\star\u_{L-1})\star\u_{L-2}\big)\ldots\star\u_{1}\right)^{\downarrow}$, where $\u^{\downarrow}$ denotes the flipped vector corresponding to $\u$, given by ${\u}^{\downarrow}[k]=\u{[D-k-1]}$ for $k=0,1,\ldots,D-1$.  

\textbf{Separable linear classification}\quad Consider a binary classification dataset $\{(\x_{n},y_{n}):n=1,2,\ldots N\}$ with  $\x_n\in\bR^D$ and  $y_n\in\{-1,1\}$. The empirical risk minimization objective for training a linear network parameterized as $\P(\u)$ is  given as follows,
\begin{equation}
\min_{\u\in\c{W}}\;\c{L}_{\P}(\u) := \sum_{n=1}^{N} \ell(\innerprod{\x_n}{\P(\u)},y_n),
\label{eq:lmu}
\end{equation}
where $\ell:\bR\times \{-1,1\}\to\bR_+$ is some surrogate loss for classification accuracy, \eg logistic loss  $\ell(\hat{y},y)=\log(1+\exp(-\hat{y}y))$ and exponential loss $\ell(\hat{y},y)= \exp(-\hat{y}y)$.

It is easy to see that both fully connected and convolutional networks of any depth $L$ can realize any linear predictor $\w\in\bR^D$.  The model class expressed by both networks is therefore simply the unconstrained class of linear predictors, and the two architectures are merely different (over) parameterizations of this class
\[\{\P_{full}(\u):\u=[\u_l\in\bR^{D_{l-1}\times D_l}]_{l=1}^L\}=\{\P_{conv}(\u):\u=[\u_l\in\bR^{D}]_{l=1}^L\}=\bR^D.\] 
Thus, the empirical risk minimization problem in \eqref{eq:lmu} is equivalent to the following optimization over the linear predictors $\w=\P(\u)$: 
\begin{equation}
\min_{\w\in\bR^D}\;\c{L}(\w):=\sum_{n=1}^{N} \ell(\innerprod{\x_n}{\w},y_n).
\label{eq:lm}
\end{equation}
Although the optimization problems \eqref{eq:lmu} and \eqref{eq:lm} are exactly equivalent in terms of the set of global minima, in this paper, we show that optimizing \eqref{eq:lmu} with different parameterizations  leads  to very different classifiers compared to optimizing \eqref{eq:lm} directly.  

In particular, consider problems \eqref{eq:lmu}/\eqref{eq:lm} on a linearly separable dataset $\{\x_{n},y_{n}\}_{n=1}^N$ and using the logistic loss (the two class version of the cross entropy loss typically used in deep learning). The global infimum of $\c{L}(\w)$ is $0$, but this is not attainable by any finite $\w$. Instead, the loss can be minimized by scaling the norm of any linear predictor that separates the data to infinity.  Thus, any sequence of predictors $\w[t]$ (say, from an optimization algorithm) that asymptotically minimizes the loss in eq. \eqref{eq:lm} necessarily separates the data and diverges in norm,  $\|{\w[t]}\|\to\infty$.  
In general there are many linear separators that correctly label the training data, each corresponding to a direction in which we can minimize \eqref{eq:lm}.  Which of these separators will we converge to when optimizing \eqref{eq:lmu}/\eqref{eq:lm}?  In other words, what is the direction $\bar\w^\infty=\lim\limits_{t\to\infty}\frac{\w[t]}{\|{\w[t]}\|}$ the iterates of our optimization algorithm will diverge in?  If this limit exist we say that $\w[t]$ {\em converges in direction} to the {\em limit direction} $\bar\w^\infty$.

\citet{soudry2017implicit} studied this implicit bias of gradient descent on \eqref{eq:lm} over the direct parameterization of $\w$.  They showed that for any linearly separable dataset and any initialization, gradient descent w.r.t.~$\w$ converges in direction to hard margin support vector machine solution:
\begin{equation}
\bar\w^\infty=\frac{\w^*_{\ell_2}}{\norm{\w^*_{\ell_2}}}\text{, where }\w^*_{\ell_2}=\argmin_{\w\in\bR^D}\norm{\w}_2^2\;\st\; \forall n, y_n\innerprod{\w}{\x_n}\ge1.
\label{eq:gd-logistic}
\end{equation}

In this paper we study the behavior of gradient descent on the problem \eqref{eq:lmu} w.r.t~different parameterizations of the model class of linear predictors. For  initialization $\u[0]$ and sequence of step sizes $\{\eta_t\}$,  gradient descent updates for  \eqref{eq:lmu} are given by, 
\begin{equation}
\u[t+1]=\u[t]-\eta_t\nabla_{\u}\c{L}_{\P}(\u[t])=\u[t]-\eta_t\nabla_{\u}{\P}(\u[t])\nabla_{\w}\c{L}(\P(\u(t))),
\label{eq:gd}
\end{equation}
where $\nabla_{\u}{\P}(.)$ denotes the Jacobian of $\P:\c{W}\to\bR^D$ with respect to the parameters $\u$, and $\nabla_{\w}\c{L}(.)$ is the gradient of the loss function in \eqref{eq:lm}.

For separable datasets, if $\u[t]$ minimizes \eqref{eq:lmu} for linear fully connected or convolutional networks, then we will again have $\norm{\u[t]}\to\infty$, and the question we ask is: what is the limit direction $\bar\w^\infty= \lim\limits_{t\to\infty}\frac{\P(\u[t])}{\|\P(\u[t])\|}$ of the predictors $\P(\u[t])$ along the optimization path?

The result in \citet{soudry2017implicit}  holds for  any loss function $\ell(u,y)$ that is strictly monotone in $uy$  with specific tail behavior, name the tightly exponential tail, which is satisfied by popular classification losses like logistic and exponential loss. In the rest of the paper, for simplicity we exclusively focus on the exponential loss function $\ell(u,y)=\exp(-uy)$,  which has the same tail behavior as that of the logistic loss. Along the lines of \citet{soudry2017implicit}, our results should also extend for any strictly monotonic loss function with a tight exponential tail, including logistic loss.  

 \begin{figure}[t!]
    \begin{center}
    \begin{subfigure}[b]{0.6\textwidth}
        \centering
        \includegraphics[width=\textwidth]{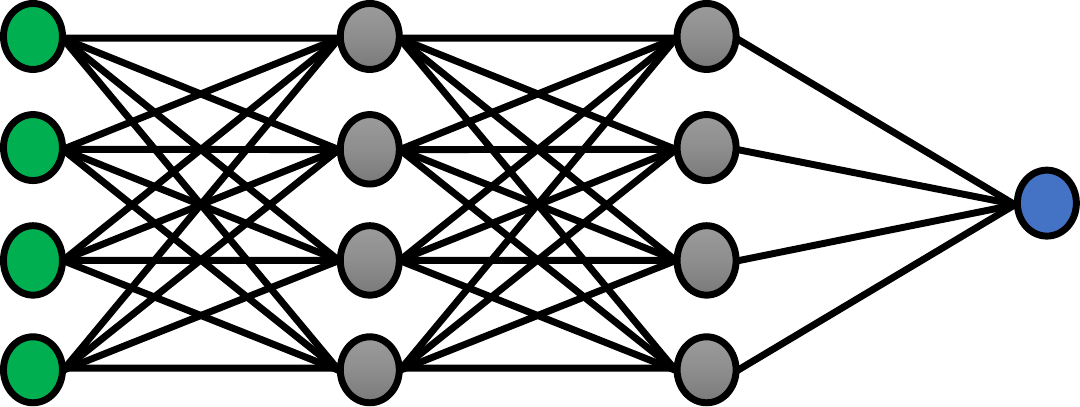}
        \captionsetup{justification=centering}
        \caption{Fully connected network of depth $L$\newline       
       $\bar\w^\infty\propto \argmin\limits_{\forall n,\,y_n\innerprod{\x_n}{\w}\ge1}\norm{\w}_2$ (independent of $L$) \label{fig:fcn}}
    \end{subfigure}
   \\  
    \begin{subfigure}[b]{0.6\textwidth}
        \centering
        \includegraphics[width=\textwidth]{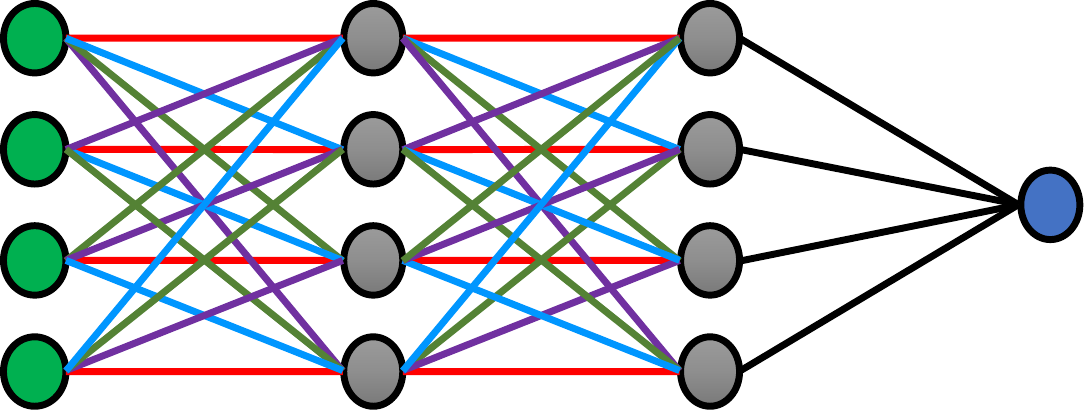}
        \captionsetup{justification=centering}
        \caption{Convolutional network  of depth $L$ \newline $\bar\w^\infty\propto$ first order stationary point of $\argmin\limits_{\forall n,\,y_n\innerprod{\x_n}{\w}\ge1}\norm{\hat{\w}}_{\nicefrac{2}{L}}$\label{fig:cn}}
    \end{subfigure}\\
    \begin{subfigure}[b]{0.6\textwidth}
    \includegraphics[width=\textwidth]{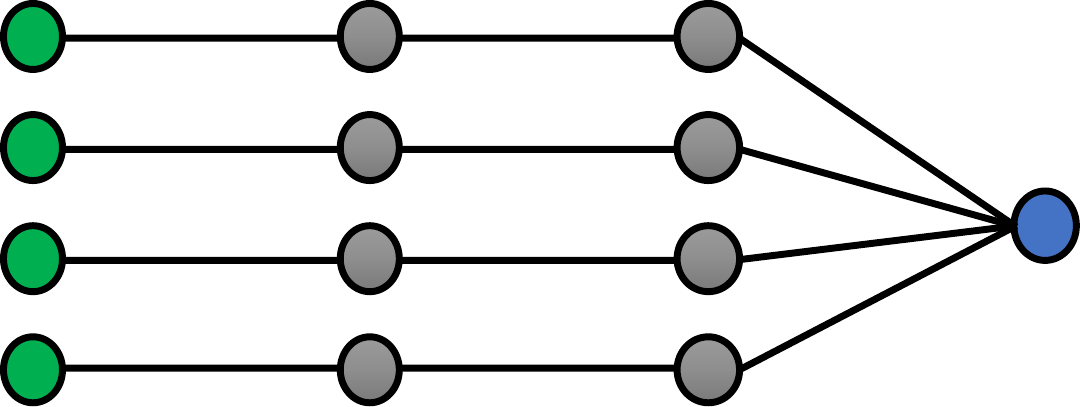}
    \captionsetup{justification=centering}
    \caption{Diagonal network  of depth $L$\newline $\bar\w^\infty\propto$ first order stationary point of$\argmin\limits_{\forall n,\,y_n\innerprod{\x_n}{\w}\ge1}\norm{\w}_{\nicefrac{2}{L}}$\label{fig:dn}}
    \end{subfigure}
    \end{center}
    \caption{Implicit bias of gradient descent for different linear network architectures. 
    }
\end{figure}
\section{Main Results}\label{sec:main}
Our main results characterize the implicit bias of gradient descent for multi-layer fully connected and convolutional networks with linear activations. For the gradient descent iterates $\u[t]$ in eq. \eqref{eq:gd}, we henceforth denote the induced linear predictor as $\w[t]=\c{P}(\u[t])$. 
\begin{assumption*} In the following theorems, we characterize the limiting predictor $\bar\w^\infty= \lim\limits_{t\to\infty}\frac{\w[t]}{\|\w[t]\|}$ under the following assumptions:
\begin{compactenum}
\item $\u[t]$  minimize the objective, \ie $\c{L}_{\P}(\u[t])\to0$.
\item $\u[t]$, and consequently $\w[t]=\P(\u[t])$, converge in direction to yield a separator $\bar\w^\infty= \lim\limits_{t\to\infty}\frac{\w[t]}{\|\w[t]\|}$ with positive margin, \ie $\min_ny_n\innerprod{\x_n}{\bar\w^\infty}>0$.
\item Gradients with respect to linear predictors $\nabla_{\w}\c{L}(\w[t])$ converge in direction.
\end{compactenum}
\end{assumption*}
These assumptions allow us to focus on the question of \textit{which specific linear predictor do gradient descent iterates converge to} by separating it from the related optimization questions of {when  gradient descent iterates minimize the non-convex objective in eq. \eqref{eq:lm} and  nicely converge in direction}.

\begin{restatable} [Linear fully connected networks]{theorem}{thmfcn} \label{thm:fcn}
For any depth $L$, almost all linearly separable datasets $\{\x_n,y_n\}_{n=1}^N$, almost all  initializations  $\u[0]$, and any bounded sequence of step sizes $\{\eta_t\}_t$, consider the sequence gradient descent iterates $\u[t]$ in eq. \eqref{eq:gd} for minimizing   $\c{L}_{\P_{full}}(\u)$ in eq. \eqref{eq:lmu} with exponential loss $\ell(\hat{y},y)=\exp(-\hat{y}y)$ over  $L$--layer fully connected linear networks. 

If \begin{inparaenum}[(a)]\item the iterates $\u[t]$  minimize the objective, \ie $\c{L}_{\P_{full}}(\u[t])\to0$, \item $\u[t]$, and consequently $\w[t]=\P_{full}(\u[t])$, converge in direction to yield a separator with positive margin, and \item gradients with respect to linear predictors $\nabla_{\w}\c{L}(\w[t])$ converge in direction, \end{inparaenum}  then the limit direction is given by,
\begin{equation}
\bar\w^\infty=\lim\limits_{t\to\infty}\frac{\P_{full}(\u[t])}{\norm{\P_{full}(\u[t])}}=\frac{\w^*_{\ell_2}}{\norm{\w^*_{\ell_2}}}\text{, where  }{\w}^*_{\ell_2}:=\argmin_{w}\norm{\w}_2^2\;\st\; \forall n, y_n\innerprod{\x_n}{\w}\ge1. 
\label{eq:gd-fcn}
\end{equation}
\end{restatable}

 For fully connected networks with single output, Theorem~\ref{thm:fcn} shows that there is no effect of depth on the  implicit bias of gradient descent. Regardless of the depth of the network, the asymptotic classifier is always the hard margin support vector machine classifier, which is also the limit direction of gradient descent for linear logistic regression with the direct parameterization of $\w=\u$.
 
 In contrast, next we show that for convolutional networks we get very different biases. Let us first look at a $2$--layer linear convolutional network, \ie a network with single convolutional layer  followed by a fully connected final layer.

Recall that $\hat{\w}\in\b{C}^D$ denote the  Fourier coefficients of ${\w}$, \ie 
$\hat{\w}[d]=\frac{1}{\sqrt{D}}\sum\limits_{p=0}^{D-1} \w{[p]}\exp(-\frac{2\pi\ci pd}{D})$, and  that any non-zero $z\in\b{C}$ is denoted in polar form as $z=|z|\e^{\ci\phi_{z}}$ for $\phi_z\in[0,2\pi)$.  
Linear predictors induced by gradient descent iterates $\u[t]$ for convolutional networks are denoted by $\w[t]=\P_{conv}(\u[t])$. 
It is evident that if $\w[t]$ converges in direction to $\bar{\w}^\infty$, then its Fourier transformation $\hat{\w}^{(t)}$ converges in direction to $\hat{\bar{\w}}^\infty$. 
In the following theorems, in addition to the earlier assumptions, we further assume a technical condition that the phase of the Fourier coefficients $\e^{\ci\phi_{\hat{\w}^{(t)}}}$ converge coordinate-wise. For coordinates $d$ with $\hat{\bar{\w}}^\infty[d]\neq0$ this follows from convergence in direction of $\u[t]$, in which case $\e^{\ci\phi_{\hat{\w}^{(t)}[d]}}\to\e^{\ci\phi_{\hat{\bar\w}^{\infty}[d]}}$. We assume such a $\phi_{\hat{\bar\w}^{\infty}[d]}$ also exists when $\hat{\bar{\w}}^\infty[d]=0$. 
\begin{restatable}[Linear convolutional networks of depth two]{theorem}{thmcn}\label{thm:conv-net}
For  almost all linearly separable datasets $\{\x_n,y_n\}_{n=1}^N$, almost all  initializations  $\u[0]$, and any  sequence of step sizes $\{\eta_t\}_t$ with $\eta_t$ smaller than the local Lipschitz at $\u[t]$, consider the sequence gradient descent iterates $\u[t]$  in eq. \eqref{eq:gd} for minimizing  $\c{L}_{\P_{conv}}(\u)$ in eq. \eqref{eq:lmu} with exponential loss over  $2$--layer linear convolutional networks.

If \begin{inparaenum}[(a)]\item the iterates $\u[t]$  minimize the objective, \ie $\c{L}_{\P_{conv}}(\u[t])\to0$, \item $\u[t]$ converge in direction to yield a separator  ${\bar\w}^{\infty}$ with positive margin, \item the phase of the Fourier coefficients $\hat{\w}^{(t)}$ of the linear predictors $\w[t]$ converge coordinate-wise, \ie $\forall d$, $\e^{\ci\phi_{\hat{\w}^{(t)}[d]}}\to\e^{\ci\phi_{\hat{\hat{\w}}^{\infty}[d]}}$, and \item the gradients $\nabla_{\w}{\c{L}}(\w[t])$ converge in direction, \end{inparaenum} then  the limit direction ${\bar\w}^{\infty}$ is given by,
\begin{equation}
\bar\w^\infty=\frac{\w^*_{\c{F},1}}{\norm{\w^*_{\c{F},1}}}\text{, where  }{\w}^*_{\c{F},1}:=\argmin_{\w} \norm{\hat{\w}}_{1}\;\st\; \forall n,\, y_n\innerprod{\w}{\x_n}\ge1.
\label{eq:gd-cn}
\end{equation} 
\end{restatable}
We already see how introducing a single convolutional layer changes the implicit bias of gradient descent---even without any explicit regularization, gradient descent on the parameters of convolutional network architecture returns solutions that are biased to have sparsity in the frequency domain. 

Furthermore,   unlike fully connected networks, for convolutional networks we also see that the implicit bias changes with the depth of the network as shown by the following theorem.   
\begin{restatable} [Linear Convolutional Networks of any Depth]{subtheorem}{thmcnl} \label{thm:conv-net-l}
  For any depth $L$, under the conditions of Theorem~\ref{thm:conv-net}, 
the limit direction $\bar\w^\infty=\lim\limits_{t\to\infty}\frac{\P_{conv}(\u[t])}{\norm{\P_{conv}(\u[t])}}$ is a scaling of  a first order stationary point of the following optimization problem,
\begin{equation}
\min_{\w}\; \norm{\hat{\w}}_{\nicefrac{2}{L}}\st \forall n,\, y_n\innerprod{\w}{\x_n}\ge1,
\label{eq:gd-cn-l}
\end{equation}
where the $\ell_p$ penalty given by $\|z\|_p=\left(\sum_{i=1}^{D} |z[i]|^p\right)^{\nicefrac{1}{p}}$   (also called the bridge penalty)  is a norm for $p=1$ and a quasi-norm for $p<1$. 
\end{restatable}

When $L>2$, and thus $p=\nicefrac{2}{L}<1$, problem \eqref{eq:gd-cn-l} is non-convex and intractable \cite{ge2011note}.  Hence, we cannot expect to ensure convergence to a global minimum.  Instead we show convergence to a first order stationary point of \eqref{eq:gd-cn-l} in the sense of sub-stationary points of \citet{rockafellar1979directionally} for optimization problems with non-smooth and non-convex objectives. These are  solutions where the \textit{local directional derivative} along the directions in the tangent cone of the constraints are all zero. 

The first order stationary points, or sub-stationary points, of \eqref{eq:gd-cn-l} are the set of  feasible predictors $\w$ such that $\exists \{\alpha_n\ge 0\}_{n=1}^N$ satisfying the following:  
$\forall n$, $y_n\innerprod{\x_n}{\w}>1\implies \alpha_n=0$, and
 \begin{equation}
 \sum_n\alpha_ny_n\hat{\x}_n\in \partial^{\circ} \norm{\hat{\w}}_p,
 \label{eq:label3}
 \end{equation} 
where $\hat{\x}_n$ is the Fourier transformation of $\x_n$, and $\partial^\circ$ denotes the local sub-differential (or Clarke's sub-differential) operator defined as 
$\partial^\circ f(\w)=\text{conv}\{\b{v}:\exists (\z_k)_k\st \z_k\to\w\tand \nabla f(\z_k)\to \b{v}\}.$

For $p=1$ and $\hat{\w}$ represented in polar form as $\hat{\w}=|\hat{\w}|\e^{\ci\boldsymbol{\phi}_{\hat{\w}}}\in\mathbb{C}^D$,  $\norm{\hat{\w}}_p$ is convex and  the local sub-differential is indeed the global sub-differential given by,
  \begin{equation}
\partial^\circ\norm{\hat{\w}}_1=\{\hat\z:\forall d,\; |\hat\z{[d]}|\le 1 \tand \hat{\w}{[d]}\neq0\implies \hat\z{[d]}= \e^{\ci\boldsymbol{\phi}_{\hat{\w}}[d]}\}.
\label{eq:label2}
\end{equation} 

For $p<1$, the local sub-differential of  $\norm{\hat{\w}}_p$ is given by, \begin{equation}
\forall p<1,\quad \partial^\circ\norm{\hat{\w}}_p=\{\hat\z:\hat{\w}{[d]}\neq0\implies \hat\z{[d]}=p\; \e^{\ci\boldsymbol{\phi}_{\hat{\w}}[d]}\;|{\hat{\w}}[d]|^{p-1}\}.
\label{eq:label1}
\end{equation} 
Figures~\ref{fig:fcn}--\ref{fig:cn} summarize the implications of the main results in the paper. 
The proof of this Theorem, exploits the following representation of $\P_{conv}(\w)$ in the Fourier domain. 
\begin{restatable}{lemma}{lemfft} \label{lem:fft-conv}For full-dimensional convolutions, $\w=\P_{conv}(\u)$ is equivalent to 
\[\hat{\w}=\text{diag}(\hat\u_1)\ldots\text{diag}(\hat\u_{L-1})\hat\u_L,\]
where for $l=1,2,\ldots,L$, $\hat\u_1\in\b{C}^D$ are the Fourier coefficients of the parameters $\u_l\in\bR^D$.
\end{restatable}
From above lemma (proved in Appendix $C$), we can see a connection of convolutional networks to a special network where the linear transformation between layers is restricted to diagonal entries (see depiction in Figure~\ref{fig:dn}), we refer to such networks as \textit{linear diagonal network}. 

The proof of Theorem~\ref{thm:fcn} and Theorem~\ref{thm:conv-net}-\ref{thm:conv-net-l} are provided in Appendix $B$ and $C$, respectively.

\section{Understanding Gradient Descent in the Parameter Space}\label{sec:l2u}
We can decompose the characterization of implicit bias of gradient descent on a parameterization $\P(\u)$ into two parts: \begin{inparaenum}[(a)]\item what is the  implicit bias of gradient descent in the space of parameters $\u$?, and \item what does this imply in term of the linear predictor $\w=\P(\u)$, \ie how does the bias in parameter space translate to the linear predictor learned from the model class? \end{inparaenum}  

We look at the first question for a broad class of linear models,  where the linear predictor is given by a homogeneous polynomial mapping of the parameters: $\w=\P(\u)$, where $\u\in\bR^P$ are the parameters of the model and $\P:\bR^P\to\bR^D$ satisfies definition below. This class covers the linear convolutional, fully connected networks, and diagonal networks discussed in Section~\ref{sec:main}.
\begin{definition*}[Homogeneous Polynomial] A multivariate polynomial function $\P:\bR^P\to \bR^D$ is said to be homogeneous, if for some finite integer $\nu<\infty$,  $\forall \alpha\in\bR,\mathbf{v}\in\bR^P$, $\P(\alpha \mathbf{v})=\alpha^\nu \P(\mathbf{v})$. %The integer $\nu$ is called the degree of  $\P$.
\end{definition*}
\begin{restatable}[Homogeneous Polynomial Parameterization]{theorem}{metathm} \label{thm:metathm} For any homogeneous polynomial map $\P:\bR^P\to\bR^D$ from parameters $\u\in\bR^D$ to linear predictors,  almost all datasets $\{\x_n,y_n\}_{n=1}^N$ separable by $\c{B}:=\{\P(\u):\u\in\bR^P\}$, almost all initializations $\u[0]$, and any bounded sequence of  step sizes $\{\eta_t\}_t$, consider the sequence of gradient descent updates $\u[t]$ from eq.~\eqref{eq:gd} for minimizing the empirical risk objective $\c{L}_{\P}(\u)$ in \eqref{eq:lmu} with exponential loss $\ell(u,y)=\exp(-uy)$. 

If  \begin{inparaenum}[(a)] \item the iterates $\u[t]$ asymptotically minimize the objective, \ie $\c{L}_{\P}(\u[t])=\c{L}(\P(\u[t]))\to0$, \item $\u[t]$, and consequently $\w[t]=\P(\u[t])$, converge in direction to yield a separator with positive margin, and \item the gradients w.r.t. to the linear predictors, $\nabla_{\w}{\c{L}}(\w[t])$ converge in direction, \end{inparaenum} then  the  limit direction  of the parameters $\bar\u^\infty=\lim\limits_{t\to\infty}\frac{\u[t]}{\norm{\u[t]}_2}$  is a positive scaling of a first order stationary point of the following  optimization problem,
\begin{equation}
\min_{\u\in\bR^P} \norm{\u}_2^2\quad\st\quad \forall n,\,y_n\innerprod{\x_n}{\P(\u)}\ge 1.
\label{eq:l2bias}
\end{equation}
\end{restatable}
Theorem~\ref{thm:metathm} is proved in Appendix $A$. 
The proof of Theorem~\ref{thm:metathm} involves showing that the asymptotic direction of gradient descent iterates satisfies the KKT conditions for first order stationary points of \eqref{eq:l2bias}. This crucially relies on two properties. First,   the sequence of gradients  $\nabla_{\w}{\c{L}}(\w[t])$ converge in direction to a positive span of support vectors of $\bar\w^\infty=\lim\limits_{t\to\infty}\frac{\w[t]}{\norm{\w[t]}}$ (Lemma $8$ in \citet{gunasekar2018characterizing}), and this result relies on the loss function $\ell$ being exponential tailed. Secondly, if $\P$ is not homogeneous, then the optimization problems $\min_{\u} \norm{\u}_2^2 \st \forall n,\innerprod{\x_n}{y_n}\ge \gamma$ for different values of unnormalized margin $\gamma$ are not equivalent and lead to different separators. Thus, for general non-homogeneous $\P$, the unnormalized margin of one does not have a significance and the necessary conditions for the first order stationarity  of  \eqref{eq:l2bias} are not satisfied.

Finally, we also note that in many cases (including linear convolutional networks) the optimization problem  \eqref{eq:l2bias} is non-convex and intractable (see \eg \cite{ge2011note}). So we cannot expect $\bar\u^\infty$ to be always be a global minimizer of eq. \eqref{eq:l2bias}. We however suspect that it is possible to obtain a stronger result that $\bar\u^\infty$ reaches a higher order stationary point or even a local minimum of the explicitly regularized estimator in eq. \eqref{eq:l2bias}. 
 
 \paragraph{Implications of the implicit bias in predictor space} While eq. \eqref{eq:l2bias} characterizes the bias of gradient descent in the parameter space, what we really care about is the effective bias introduced in the space of functions learned by the network. In our case, this class of functions is the set of linear predictors $\{\w\in\bR^D\}$.  
 The $\ell_2$ norm penalized solution in eq. \eqref{eq:l2bias}, is equivalently given by,
\begin{equation}
\w^*_{\c{R}_{\P}}=\argmin_{\w} \c{R}_{\P}(\w) \st\forall n,\,y_n\innerprod{\w}{\x_n}\ge 1,\text{ where }\c{R}_{\P}(\w)=\inf_{\u:\P(\u)=\w} \norm{\u}_2^2.
\label{eq:Rwmargin}
\end{equation}
The problems in eq. \eqref{eq:l2bias} and eq. \eqref{eq:Rwmargin} have the same global minimizers, \ie  $\u^*$ is global minimizer of eq. \eqref{eq:l2bias} if and only if $\w^*=\P(\u^*)$ minimizes  eq. \eqref{eq:Rwmargin}. However, such an equivalence does not extend to the  stationary points of the two problems. Specifically, it is possible that a stationary point of eq. \eqref{eq:l2bias} is merely a feasible point for eq. \eqref{eq:Rwmargin} with no special significance.  
So instead of using Theorem~\ref{thm:metathm}, for the specific networks  in Section~\ref{sec:main}, we directly show (in Appendix) that  gradient descent updates converge in direction to  a first order stationary point of the problem in eq. \eqref{eq:Rwmargin}. 

\section{Understanding Gradient Descent in Predictor Space}\label{sec:Rw}
In the previous section, we saw that the implicit bias of gradient descent on a parameterization $\P(\u)$  can be described in terms of the optimization problem \eqref{eq:l2bias}  and the implied penalty function $\c{R}_{\P}(\w)=\min_{\u:\P(\u)=\w} \norm{\u}_2^2$.  We now turn to studying this implied penalty $\c{R}_{\P}(\w)$ and obtaining explicit forms for it, which will reveal the precise form of the implicit bias in terms of the learned linear predictor. The proofs of the lemmas in this section are provided in the Appendix $D$.

\begin{restatable}{lemma}{lemfcn} \label{prop:fcn} For fully connected networks of any depth $L>0$, 
\[\c{R}_{\P_{full}}(\w)=\min_{\u:\P_{full}(\u)=\w} \norm{\u}_2^2=L \norm{\w}^{\nicefrac{2}{L}}_2=\textrm{monotone}(\norm{\w}_2).\]
\end{restatable}
We see that $\w^*_{\c{R}_{\P_{full}}}=\argmin_{\w}  \c{R}_{\P_{full}}(\w)\st \forall n, y_n\innerprod{\x_n}{\w}\ge 1$  in eq.~\eqref{eq:Rwmargin} for fully connected networks is independent of the depth of the network $L$. In Theorem~\ref{thm:fcn}, we indeed show that gradient descent for this class of networks converges in the direction of $\w^*_{\c{R}_{\P_{full}}}$. 

Next, we motivate the characterization of $\c{R}_{\P}(\w)$ for linear convolutional networks by first looking at the special  \textit{linear diagonal network} depicted in Figure~\ref{fig:dn}. 
The depth--$L$ diagonal network is parameterized by $\u=[\u_l\in\bR^D]_{l=1}^L$ and the mapping to a linear predictor is given by ${\c{P}_{diag}(\u)=\text{diag}(\u_1)\text{diag}(\u_2)\ldots\text{diag}(\u_{L-1})\u_L}$. 
\begin{restatable}{lemma}{lemdn} For a depth--$L$ diagonal network with parameters   $\u=[\u_l\in\bR^D]_{l-1}^{L}$, we have
\[\c{R}_{\P_{diag}}(\w)=\min_{\u:\P_{diag}(\u)=\w} \norm{\u}_2^2=L\norm{\w}_{\nicefrac{2}{L}}^{\nicefrac{2}{L}}=\textrm{monotone}(\norm{\w}_{\nicefrac{2}{L}}).\]
\end{restatable}

Finally, for full width linear convolutional networks parameterized by $\u=[\u_l\in \bR^D]_{l=1}^L$, recall the following representation of $\w=\P_{conv}(\u)$ in Fourier from Lemma~\ref{lem:fft-conv}. 
\[\hat{\w}=\text{diag}(\hat\u_1)\ldots\text{diag}(\hat\u_{L-1})\hat\u_L,\]
where $\hat{\w},\hat{\u}_l\in\b{C}^D$ are Fourier basis representation of $\w,\u_l\in\bR^D$, respectively. 
Extending the result of diagonal networks for the complex vector spaces, we get the following characterization of $\c{R}_{\P_{conv}}(\w)$ for linear convolutional networks.
\begin{restatable}{lemma}{lemcn} For a depth--$L$  convolutional network with parameters   $\u=[\u_l\in\bR^D]_{l-1}^{L}$, we have %there exists a monotonic function $\sigma:\bR\to\bR$ such that 
\[\c{R}_{\P_{conv}}(\w)=\min_{\u:\P_{conv}(\u)=\w} \norm{\u}_2^2=L\|{\hat\w}\|_{\nicefrac{2}{L}}^{\nicefrac{2}{L}}=\textrm{monotone}(\norm{\hat{\w}}_{\nicefrac{2}{L}}).\]
\end{restatable}

\section{Discussion}
In this paper, we characterized the implicit bias of gradient descent on linear convolutional networks.  We showed that even in the case of  linear activations and a full width convolution, wherein the convolutional network defines the exact same model class as fully connected networks, merely changing to a convolutional parameterization  introduces radically different, and very interesting, bias when training with gradient descent.  Namely, training a convolutional representation with gradient descent implicitly biases  towards sparsity in the frequency domain representation of linear predictor.  

For convenience and simplicity of presentation, we studied one dimensional circular convolutions.  Our results can be directly extended to higher dimensional input signals and convolutions, including the two-dimensional convolutions common in image processing and computer vision.  We also expect similar results for convolutions with zero padding instead of circular convolutions, although this requires more care with analysis of the edge effects.

A more significant way in which our setup differs from usual convolutional networks is that we use full width convolutions, while in practice it is common to use convolutions with bounded width, much smaller then the input dimensionality.  This setting is within the scope of Theorem~\ref{thm:metathm}, as the linear transformation is still homogeneous.  However, understanding the implied bias in the predictor space, i.e.~understanding $\c{R}_{\P}(\w)$ requires additional work. It will be very interesting to see if restricting the width of the convolutional network gives rise to further interesting behaviors.

Another important direction for future study is understanding the implicit bias for networks with multiple outputs.  For both fully connected and convolutional networks, we looked at networks with a single output.  With $C>1$ outputs, the network implements a linear transformation $\x\mapsto\w \x$ where $\w\in\bR^{C\times D}$ is now a matrix. Results for matrix sensing in  \citet{gunasekar2018characterizing} imply that for two layer fully connected networks with multiple outputs, the implicit bias is to a maximum margin solution with respect to the nuclear norm $\norm{\w}_\star$.  This is already different from the implicit bias of a one-layer ``network'' (i.e.~optimizing $\w$ directly), which would be in terms of the Frobenius norm $\norm{\w}_F$ (from the result of \citet{soudry2017implicit}).  We suspect that with multiple outputs, as more layers are added, even fully connected networks  exhibit a shrinking sparsity penalty on the singular values of the effective linear matrix predictor $\w\in\bR^{C\times D}$.  Precisely characterizing these biases requires further study.

When using convolutions as part of a larger network, with multiple parallel filters, max pooling, and non-linear activations, the situation is of course more complex, and we do not expect to get the exact same bias.  However, we do expect the bias to be at the very least related to the sparsity-in-frequency-domain bias that we uncover here, and we hope our work can serve as a basis for further such study.  There are of course many other implicit and explicit sources of inductive bias---here we show that merely parameterizing transformations via convolutions and using gradient descent for training already induces sparsity in the frequency domain.

On a technical level, we provided a generic characterization for the bias of gradient descent on linear models parameterized as $\w=\P(\u)$ for a homogeneous polynomial $\P$.  The $\ell_2$ bias (in parameter space) we obtained is not surprising, but also should not be taken for granted -- \eg the result does not hold in general for non-homogeneous $\P$, and even with homogeneous polynomials, the characterization is not as crisp when other loss functions are used, \eg~with a squared loss and matrix factorization (a homogeneous degree two polynomial representation), the implicit bias is much more fragile \cite{gunasekar2017implicit,li2017algorithmic}.  Moreover, Theorem \ref{thm:metathm} only ensures convergence to first order stationary point in the parameter space, which is not sufficient for convergence to stationary points of the implied bias in the model space (eq. \eqref{eq:Rwmargin}). It is of interest for future work to strengthen this result to show either convergence to higher order stationary points or local minima in parameter space, or to directly show the convergence to stationary points of \eqref{eq:Rwmargin}. 

It would also be of interest to strengthen other technical aspects of our results: extend the results to loss functions with tight exponential tails (including logistic loss) and handle all datasets including the set of measure zero degenerate datasets---these should be possible following the techniques of \citet{soudry2017implicit,telgarsky2013margins,ji2018risk}. We can also calculate exact rates of convergence to  the asymptotic separator along the lines of \citet{soudry2017implicit,nacson2018convergence,ji2018risk} showing how fast the inductive bias from optimization kicks in and why it might be beneficial to continue optimizing even after the loss value $\c{L}(\w[t])$ itself is negligible. 
Finally, for logistic regression, \citet{ji2018risk} extend the results of asymptotic convergence of gradient descent classifier to the cases where the data is not strictly linearly separable. This is an important relaxation of our assumption on strict linear separability. More generally, for non-separable data, we would like a more fine grained analysis connecting the iterates $\w^{(t)}$ along the optimization path to the estimates along regularization path, $\hat{\w}(c)=\argmin_{\c{R}_{\P}(\w)\le c}\c{L}(\w)$, where an explicit regularization is added to the optimization objective.

{\small
\bibliographystyle{plainnat}
\bibliography{suriya}}
\clearpage
{
\appendix
\begin{center}\Large \textbf Appendix 
\end{center}
The proofs of the  theorems in the paper are organized as follows: In Appendix \ref{app-hom} we first give the proof for Theorem~\ref{thm:metathm}, which includes linear fully connected and full width convolutional networks  as special cases. This gives us some general results that can be special-cased to prove the stronger results for these networks in Section~\ref{sec:main}. 
In Appendix~\ref{app-fcn}, we prove Theorem~\ref{thm:fcn} on the implicit bias of  fully connected linear networks. In Appendix~\ref{app-conv}, we prove Theorem~\ref{thm:conv-net}--\ref{thm:conv-net-l} on the implicit bias of  linear convolutional networks. 
Finally, in  Appendix \ref{app:lem} we prove the lemmas in Section~\ref{sec:Rw} on computing the form of implicit bias of linear networks learned using gradient descent. 

Unless specified otherwise, $\norm{.}$ denotes the  Euclidean norm. We additionally use the notation  $\mathbf{v}\propto\mathbf{v'}$ to denote equality up to strictly positive scalar multipliers, \ie when $\mathbf{v}=\gamma\mathbf{v'}$ for some $\gamma>0$.

The following is a paraphrasing of Lemma~$8$ in \citet{gunasekar2018characterizing} and is used in multiple proofs.
\begin{lemma} \label{lem:grad-conv} [Lemma~$8$ in \cite{gunasekar2018characterizing}]
For almost all linearly separable dataset $\{\x_n,y_n\}_n$, consider any sequence $\w[t]$ that  minimizes the empirical objective in eq. \eqref{eq:lm}, \ie   $\c{L}(\w[t])\to0$. If \begin{inparaenum}[(a)] \item  $\bar{\w}^\infty:=\lim\limits_{t\to\infty}\frac{\w[t]}{\norm{\w[t]}}$ exists and has a positive margin, and \item $\bar{\z}^\infty:=\lim\limits_{t\to\infty}\frac{-\nabla_{\w}\c{L}(\w[t])}{\norm{\nabla_{\w}\c{L}(\w[t])}}$ exists\end{inparaenum}, then 
$\exists \{\alpha_n\ge 0\}_{n\in S} \st \bar{\z}^\infty=\sum_{n\in S}\alpha_n\,y_n\x_{n},$
 where $S=\{n:y_n\innerprod{\bar{\w}^\infty}{\x_n}=\min_n y_n\innerprod{\bar{\w}^\infty}{\x_{n}}\}$ are the indices of the data points with smallest margin to the limit direction $\bar \w^\infty$. 
\end{lemma}

\section{Homogeneous Polynomial Parameterization: Proof of Theorem~\ref{thm:metathm}}\label{app-hom}
\metathm*
\begin{proof}  $\u[t]$ are the sequence gradient descent iterates from eq. \eqref{eq:gd} for minimizing $\c{L}_{\P}(\u)$ in eq \eqref{eq:lmu} with exponential loss over the model class of $\c{B}=\{\P(\u):\u\in\bR^P\}$, where $\P$ is a homogeneous polynomial function. 
We first introduce some notation.% for $\u[t]$ and $\w[t]=\P(\u[t])$, 
\begin{asparaenum}
\item  From  the assumption in theorem, we have that  $\bar{\u}^\infty=\lim\limits_{t\to\infty}\frac{\u[t]}{\norm{\u[t]}}$. Denoting  $g(t)=\norm{\u[t]}$, we have that for some $\bdelta[t]_{\u}\to 0$, the following representation of $\u[t]$ holds. 
\begin{align}
\u[t]&=\bar{\u}^\infty g(t)+\bdelta[t]_{\u}\,g(t).\label{eq:u}
\end{align}
\item Let $\w[t]=\P(\u[t])$ denote the sequence of linear predictors for this network induced by the gradient descent iterates. We can see that $\w[t]$ converges in direction too using the following arguments: homogeneity of $\P$ implies that $\P(\nicefrac{\u[t]}{\norm{\u[t]}})=\nicefrac{\P(\u[t])}{\norm{\u[t]}^\nu}$ for some $\nu$. Hence, $\frac{\w[t]}{\norm{\w[t]}}=\frac{\P(\nicefrac{\u[t]}{\norm{\u[t]}})}{\norm{\P(\nicefrac{\u[t]}{\norm{\u[t]}})}}\overset{t\to\infty}{\rightarrow}\frac{\P(\bar \u^\infty)}{\norm{\P(\bar \u^\infty)}}:=\bar{\w}^\infty$. 

\item  $\z[t]=-\nabla_{\w}\c{L}(\w[t])=\sum_n\exp(-\innerprod{\w[t]}{y_n\x_n})y_n\x_n$. Since we assume that $\z[t]$ converges in direction, let $\bar{\z}^\infty=\lim\limits_{t\to\infty}\frac{\z[t]}{\norm{\z[t]}}$. Denoting $p(t)=\norm{\z[t]}$, for some $\bdelta[t]_{\z}\to0$, we can  write $\z[t]$ as,
\begin{align}
\z[t]=\bar{\z}^\infty p(t)+\bdelta[t]_{\z}\,p(t),
\label{eq:z}
\end{align}
\item  Let $\nabla_{\u}\P\left(\u[t]\right)\in\bR^{P\times D}$ denote the Jacobian of $\P(\u)$, \ie $\nabla_{\u}\P\left(\u[t]\right)[p,d]=\pdv{(\P\left(\u[t]\right)[d])}{\u{[p]}}$. 
If $\P:\bR^P\to \bR^D$ is a homogeneous polynomial of degree $\nu>0$, then  $\nabla_{\u}\P:\bR^P\to\bR^{P\times D}$ is a homogeneous polynomial of degree $\nu-1$.  Using eq. \eqref{eq:u}, we have 
\[\nabla_{\u}\P(\bar\u^\infty)=
\lim\limits_{t\to\infty}\nabla_{\u}\P\left(\frac{\u[t]}{g(t)}\right)=\lim\limits_{t\to\infty}\frac{\nabla_{\u}\P(\u[t])}{g(t)^{\nu-1}}\]
Thus, $\exists \bdelta[t]_{1}\to0$, such that 
\begin{align}
%\P\left(\u[t]\right)&=\P\left(\bar\u^\infty\right)g(t)^{\nu}+\bdelta[t]_{1}\,g(t)^{\nu},\\
\nabla_{\u}\P\left(\u[t]\right)&=\nabla_{\u}\P\left(\bar\u^\infty\right)\,g(t)^{\nu-1}+\bdelta[t]_{1}\,g(t)^{\nu-1}.
\label{eq:pu}
\end{align}
\item Finally, from the definition of $\nabla_{\u}\P(\u)$, we have $\nabla_{\u}\c{L}_\P(\u[t])=\nabla_{\u}\P\left(\u[t]\right)\nabla_{\w}\c{L}(\w[t])$, and hence from eq. \eqref{eq:gd}, 
\begin{equation}
\Delta\u[t]:=\u[t+1]-\u[t]=\eta_t \nabla_{\u}\P\left(\u[t]\right)\z[t]
\label{eq:dw1}
\end{equation}
\end{asparaenum}

Using the assumptions in the theorem along with our argument above for convergence of $\w[t]$ in direction, we satisfy the conditions of Lemma~\ref{lem:grad-conv}, which will be crucially used in our proof. 

\paragraph{KKT conditions for first order stationary points} 
We want show that  there exists a positive scaling of $\bar\u^\infty$, denoted as $\tilde{\u}^\infty=\gamma\bar\u^\infty$ for some $\gamma>0$,  
such that $\tilde{\u}^\infty$ is a first order stationary point of the explicitly regularized problem in eq. \eqref{eq:l2bias}. Towards this we show that $\tilde{\u}^\infty$ satisfy the following first order KKT conditions of eq. \eqref{eq:l2bias} %where without loss of generality, we assume that $y_n=1$ since for linear models, the sign of $y_n$ can be absorbed into $\x_n$. : 
\begin{equation}
\begin{split}
& \forall n,\,y_n\innerprod{\x_n}{\P(\u)} \ge 1,\\
&\exists \{\alpha_n\}_{n=1}^N \st \forall n, \alpha_n\ge0 \tand \alpha_n=0, \forall n\notin {S}:=\{n\in[N]:y_{n}\innerprod{\x_n}{\P(\u)}=1\},\\
& \u = \nabla_{\u}\P(\u)\left[\sum_n\alpha_n\,y_n\x_{n}\right].
\end{split}
\label{eq:kkt-l2}
\end{equation}

\textbf{Primal feasibility.} 
We showed earlier that if $\u[t]$ converges in direction, then $\w[t]=\P(\u[t])$ converges in direction to $\bar \w^\infty=\lim\limits_{t\to\infty}\frac{\w[t]}{\norm{\w[t]}}\propto\P(\bar\u^\infty)$. Further, from the assumptions in the theorem,  we have that 
 $\bar{\w}^\infty$ satisfies $\forall n$, $y_n\innerprod{\x_n}{\bar{\w}^\infty}>0$, which also implies $\min_ny_n\innerprod{\x_n}{\P(\bar{\u}^\infty)}>0$ since $\bar{\w}^\infty\propto \P(\bar\u^\infty)$. 
Now, if $\P$ is homogeneous of of degree $\nu$, then for $\gamma=\left({\min_ny_n\innerprod{\x_n}{\P(\bar{\u}^\infty)}}\right)^{-\nicefrac{1}{\nu}}$, $\tilde{\u}^\infty=\gamma \bar\u^\infty$ satisfies $\min_n y_n\innerprod{\x_n}{\P(\tilde{\u}^\infty)}= 1$.

\paragraph{Showing other KKT conditions for $\tilde\u^\infty$.} 
The crux of the proof of  Theorem~\ref{thm:metathm} involves showing the existence of $\{\alpha_n\ge0\}_n$ such that the stationarity and complementary slackness conditions in eq. \eqref{eq:kkt-l2} are satisfied.  This crucially relies on a key lemma (Lemma~\ref{lem:grad-conv}) showing that the gradient in the space of linear predictors $\nabla_{\w}{\c{L}}(\w[t])$ are dominated by positive linear combinations of support vectors of the asymptotic predictor $\bar\w^\infty$.

Let $S_\infty=\{n:y_n\innerprod{\P(\tilde{\u}^\infty)}{\x_n}=1\}$ denote the indices of support vectors for  $\P(\tilde{\u}^\infty)$, which are also the support vectors of $\bar\w^\infty$, since by homogeneity of $\P$, $\bar\w^\infty\propto\P(\bar{\u}^\infty)\propto\P(\tilde{\u}^\infty)$. 
Thus, from Lemma~\ref{lem:grad-conv}, we have  $\bar{\z}^\infty=\lim\limits_{t\to\infty}\frac{\z[t]}{\norm{\z[t]}}=\sum_{n\in S_\infty} {\alpha}_ny_n\x_n$ for some $\{{\alpha}_n\}_{n\in S_\infty}$ such that ${\alpha}_n\ge 0$. We propose a positive scaling of this $\{\alpha_n\}_{n=1}^N$ as our candidate dual certificate, which satisfies both dual feasibility and complementary slackness.

To prove the theorem, the remaining step is to show that $\tilde{\u}^\infty\propto \nabla_{\u}\P(\tilde\u^\infty)\bar\z^\infty$. Since $\tilde{\u}^\infty=\gamma \bar\u^\infty$ and $\P$ is homogeneous, this condition  is equivalent to showing that $\bar{\u}^\infty\propto \nabla_{\u}\c{P}(\bar{\u}^\infty)\bar\z^\infty$.% for any positive scalar $D$. 

\paragraph{Showing that $\bar{\u}^\infty\propto \nabla_{\u}\P(\bar{\u}^\infty)\bar\z^\infty$.}
 
Substituting for  $\z[t]$ and $\nabla_{\u}\P(\u[t])$  from  eqs. \eqref{eq:z} and \eqref{eq:pu}, respectively, in the gradient descent updates (eq. \eqref{eq:dw1}),  we have the following:
\begin{equation}
\begin{split}
\u[t+1]-\u[t]&=\eta_t\nabla_{\u}\P\left(\u[t]\right)\z[t]\\
&=\eta_t\left(\nabla_{\u}\P\left(\bar{\u}^\infty\right)\,g(t)^{\nu-1}+\bdelta[t]_{1}\,g(t)^{\nu-1}\right)\left(\bar{\z}^\infty p(t)+\bdelta[t]_{\z}\,p(t)\right)\\
&\overset{(a)}=\left(\eta_t{p(t)g(t)^{\nu-1}}\right)[\nabla_{\u}\P\left(\bar \u^\infty\right)\bar{\z}^\infty +\bdelta[t]],
\end{split}
\label{eq:slackness}
\end{equation}
where in $(a)$ $\bdelta[t]=\nabla_{\u}\P\left(\bar{\u}^\infty\right)\bdelta[t]_{\z}+\bdelta[t]_{1}\bdelta[t]_{\z}+\bdelta[t]_{1}\bar{\z}^\infty\to0$.

Summing over $t$, we have 
\begin{equation}
\u[t]-\u[0]=\nabla_{\u}\P\left(\bar \u^\infty\right)\bar{\z}^\infty\sum_{u<t}{\eta_u p(u)g(u)^{\nu-1}}+\sum_{u<t}\bdelta[u]{\eta_up(u)g(u)^{\nu-1}},
\label{eq:last1}
\end{equation}

We want to argue that the first term, \ie $\nabla_{\u}\P\left(\bar \u^\infty\right)\bar{\z}^\infty$, is the dominant term. Towards this we state and prove the following intermediate claim
\begin{claim} $\norm{\nabla_{\u}\P\left(\bar \u^\infty\right)\bar{\z}^\infty}>0$ and $\sum_{u<t}{\eta_up(u)g(u)^{\nu-1}}\to\infty$.
\end{claim}
\begin{proof}
First, it is straight forward to check that for any scalar valued homogeneous polynomial $f:\bR^P\to\bR$ of degree $\nu$, we have $\innerprod{\u}{\nabla_{\u}{f}(\u)}=\nu {f}(\u)$, where for $p=1,2\ldots, P$,  $\nabla_{\u}{f}(\u)[p]=\dv{{f}(\u)}{{\u}[p]}$ (this is also known as the Euler's homogeneous function theorem). Extending this to our vector valued homogeneous function $\P:\bR^P\to\bR^D$, we have that for all $\u$, the Jacobian $\nabla_{\u}\P(\u)\in\bR^{P\times D}$ satisfies $\nabla_{\u}\P(\u)^\top \u=\nu \P(\u)$. 

Moreover, we have that for the limit direction $\bar\u^\infty$, the margin of the corresponding classifier is strictly positive, \ie $\min_n y_n\innerprod{\P(\bar\u^\infty)}{\x_n}>0$. Now from Lemma~\ref{lem:grad-conv}, using that $\bar\z^\infty=\sum_{n\in S_\infty}\alpha_n y_n\x_n$ for $\alpha_n\ge0$ (and not all zero since $\bar{\z}^\infty$ is unit norm), we immediately get the following
\[\bar\u^{\infty^\top}\nabla_{\u}\P(\bar\u^\infty)\bar\z^\infty=\nu\P(\bar\u^\infty)^\top \bar\z^\infty=\nu\sum_n\alpha_n y_n\innerprod{\x_n}{\P(\u^\infty)}>0\implies\nabla_{\u}\P(\bar\u^\infty)\bar\z^\infty \neq 0.\]  

To prove the second part, we note the following
\begin{compactitem}
\item since $\bdelta[t]\to0$ in eq. \eqref{eq:last1}, $\exists t_0$ such that $\forall t>t_0$, $\norm{\bdelta[t]}\le 1$, and since all the incremental updates to gradient descent are finite, we have that $\sup_{t}\norm{\bdelta[t]}<\infty$,
\item since $p(t)=\norm{\z[t]}$ and $g(t)=\norm{\u[t]}$ are positive, we have that $b_t=\sum_{u<t}\eta_u p(u) g(u)^{\nu-1}$ is monotonic increasing. 
\end{compactitem}
Thus, if $\lim\sup_{t\to\infty} b_t=\infty$ then $\lim_{t\to\infty} b_t=\infty$. 
 On contrary, if $\lim\sup_{t\to\infty} b_t=C<\infty$, then from eq. \eqref{eq:last1}, for large $t$ we get, $\norm{\u[t]}\le\norm{\u[0]}+\norm{\nabla\P(\bar\u^\infty)\bar \z^\infty}C +\left(\sup_{t}\norm{\bdelta[t]}\right) C<\infty$ which contradicts $\norm{\u[t]}\to\infty$.
\end{proof}
From above claim, the sequence $b_t=\sum_{u<t}\eta_u p(u) g(u)^{\nu-1}$ is monotonic increasing and diverging. Thus, for $a_t=\sum_{u<t}\bdelta[u]\eta_u p(u) g(u)^{\nu-1}$, using Stolz-Cesaro theorem~(Theorem~\ref{thm:stolzcesaro}), we have  
\begin{flalign}
\nonumber\lim_{t\to\infty}\frac{a_t}{b_t}=\lim_{t\to\infty}\frac{\sum_{u<t}\bdelta[u]\eta_u p(u) g(u)^{\nu-1}}{\sum_{u<t}\eta_u p(u) g(u)^{\nu-1}}=\lim_{t\to\infty}\frac{a_{t+1}-a_{t}}{b_{t+1}-b_{t}}=\lim_{t\to\infty}\bdelta[t]=0.\\
\implies \text{for }\bdelta[t]_2\to0,\; \text{we have }\sum_{u<t}\bdelta[u]\eta_u p(u) g(u)^{\nu-1}=\bdelta[t]_2\sum_{u<t}\eta_u p(u) g(u)^{\nu-1}.
\label{eq:last2}
\end{flalign}

Substituting eq. \eqref{eq:last2} in eq. \eqref{eq:last1}, we have 
\begin{flalign}
\u[t]&\overset{(a)}=\left[\nabla_{\u}\P\left(\bar \u^\infty\right)\bar{\z}^\infty+{\bdelta[t]_3}\right]\left[\sum_{u<t}\eta_u p(u) g(u)^{\nu-1}\right]\\
\implies \frac{\u[t]}{\norm{\u[t]}}&=\frac{\nabla_{\u}\P\left(\bar \u^\infty\right)\bar{\z}^\infty+\bdelta[t]_3}{\norm{\nabla_{\u}\P\left(\bar \u^\infty\right)\bar{\z}^\infty+\bdelta[t]_3}}\overset{(b)}\to \frac{\nabla_{\u}\P\left(\bar \u^\infty\right)\bar{\z}^\infty}{\norm{\nabla_{\u}\P\left(\bar \u^\infty\right)\bar{\z}^\infty}}\\
\implies\bar \u^\infty&= \lim_{t\to\infty}\frac{\u[t]}{\norm{\u[t]}}=\frac{\nabla_{\u}\P\left(\bar \u^\infty\right)\bar{\z}^\infty}{\norm{\nabla_{\u}\P\left(\bar \u^\infty\right)\bar{\z}^\infty}}\propto\nabla_{\u}\P\left(\bar \u^\infty\right)\bar{\z}^\infty,
\label{eq:stationarity-meta}
\end{flalign}
where in $(a)$ we absorbed  the diminishing terms into  $\bdelta[t]_3=\bdelta[t]_2+\u[0]/\sum_{u<t}\eta_up(u)g(u)^{\nu-1}\to0$, $(b)$ follows since we proved in the claim above that $\nabla_{\u}\P\left(\bar \u^\infty\right)\bar{\z}^\infty\neq 0$  and hence dominates $\bdelta[t]_3$. 

We have  shown that $\bar \u^\infty=\bar{\gamma}\nabla_{\u}\P\left(\bar \u^\infty\right)\bar{\z}^\infty$ for a positive scalar $\bar{\gamma}$, which completes the proof.
\end{proof}

%\subsection{Proof Lemma~\ref{lem:grad-conv}}\label{sec:lemma-3}

\section{Linear Fully Connected Networks: Proof of Theorem~\ref{thm:fcn}}\label{app-fcn}
\thmfcn*
\begin{proof} Recall that for fully connected networks of any depth $L>0$ with parameters  
${\u=[\u_l\in\bR^{D_{l-1}\times D_l}]_{l-1}^{L}}$, the equivalent  linear predictor given by $\P_{full}(\u)=\u_1\u_2\ldots\u_L$ is a homogeneous polynomial of degree $L$. 

Let $\u[t]=[\u[t]_l\in\bR^{D_{l-1}\times D_l}]_{l=1}^L$ denote the iterates of individual matrices $\u_l$ along the gradient descent path, and $\w[t]=\P_{full}(\u[t])$ denote the corresponding sequence of linear predictors. 

We first introduce the following notation. 
\begin{compactenum}
\item Let $\bar\u^\infty=\lim\limits_{t\to \infty}\frac{\u[t]}{\norm{\u[t]}}$ denote the limit direction of the parameters, with component matrices in each layer denoted as $\bar\u^\infty=[\bar\u_l^\infty]$. Specializing \eqref{eq:u} for fully connected networks, we have:
\begin{align}
\u[t]_l&=\bar{\u}_l^\infty g(t)+\bdelta[t]_{\u_l}\,g(t), \label{eq:ufcn}
\end{align}
where $g(t)=\norm{\u[t]}$ and $\bdelta[t]_{\u_l}\to 0$. 
\item For $0<l_1< l_2\le L$, denote $\u[t]_{l_1:l_2}=\u[t]_{l_1} \u[t]_{l_1+1}\ldots\u[t]_{l_2}$ and $\bar\u_{l_1:l_2}^\infty=\bar\u_{l_1}^\infty \bar\u_{l_1+1}^\infty\ldots\bar\u_{l_2}^\infty$. Using eq. \eqref{eq:ufcn}, we can check by induction on $l_2-l_1$ that $\lim\limits_{t\to\infty}\frac{\u[t]_{l_1:l_2}}{g(t)^{l_2-l_1+1}}=\bar\u^\infty_{l_1:l_2}$, and hence $\exists\bdelta[t]_{\u_{l_1:l_2}}\to0$ such that the following holds,
\begin{align}
\u[t]_{l_1:l_2}&=\bar\u_{l_1:l_2}^\infty\, g(t)^{l_2-l_1+1}+\bdelta[t]_{\u_{l_1:l_2}}\,g(t)^{l_2-l_1+1}. 
\label{eq:ulfcn}
\end{align}
\item Let $\z[t]=-\nabla_{\w}\c{L}(\w[t])$.  Again repeating eq. \eqref{eq:z} for fully connected networks, we have for some $\bdelta[t]_{\z}\to0$ and $p(t)=\norm{\z[t]}$, 
\begin{equation}
\z[t]={\bar\z}^\infty p(t)+\bdelta[t]_{\z}\,p(t).
\label{eq:fcn-z}
\end{equation}
\item From Lemma~\ref{lem:grad-conv}, we have that $\exists \{\alpha_n\}_{n\in S_\infty}$ such that $\bar{\z}^\infty=\sum_{n\in S_\infty}\alpha_n\,y_n\x_n$, where $S_\infty$ are support vectors of $\bar{\w}^\infty=\lim\limits_{t\to\infty}\frac{\w[t]}{\norm{\w[t]}}\propto\P_{full}(\bar{\u}^\infty)$.
\end{compactenum}

The proof of Theorem~\ref{thm:fcn} is fairly straight forward from using Lemma \ref{lem:grad-conv} and the intermediate results in the proof of Theorem~\ref{thm:metathm}. 

\paragraph{Showing KKT conditions for $\tilde{\w}^\infty\propto\P_{full}(\bar{\u}^\infty)$.}
Using our notation described above, we have $\bar{\u}^\infty_{1:L}=\P_{full}(\bar{\u}^\infty)$.
In the following arguments we show that a positive scaling $\tilde{\w}^\infty=\gamma \bar{\u}^\infty_{1:L}$ satisfies the following KKT conditions for the optimality of $\ell_2$ maximum margin  problem in eq. \eqref{eq:gd-fcn}:
\begin{equation}
\begin{split}
\exists \{\alpha_n\}_{n=1}^N\quad\st\quad& \forall n,\,y_n\innerprod{\x_n}{\w} \ge 1, \w = \sum_n\alpha_n\,y_n\x_{n},\\
& \forall n, \alpha_n\ge0 \tand \alpha_n=0, \forall i\notin {S}:=\{i\in[N]:y_{n}\innerprod{\x_n}{\w}=1\}.
\end{split}
\label{eq:kkt-fcn}
\end{equation}

As we saw in proof of Theorem~\ref{thm:metathm}, since $\bar{\u}^\infty_{1:L}=\P_{full}(\bar{\u}^\infty)$ has strictly positive margin, using homogeneity of $\P_{full}$, we can scale $\bar{\u}^\infty_{1:L}$ to get $\tilde{\w}^\infty=\gamma \bar{\u}^\infty_{1:L}$ with unit margin, \ie $\forall n,\,y_n\innerprod{\x_n}{\tilde{\w}^\infty} \ge 1$.
For  dual variables, we again use a positive scaling of $\alpha_n$ from Lemma~\ref{lem:grad-conv}, such that $\bar{\z}^\infty=\sum_{n\in S_\infty}\alpha_n\,y_n\x_n$. In order to prove the theorem, we need to show that $\tilde{\w}^\infty\propto \bar{\z}^\infty$ or equivalently $\bar{\u}^\infty_{1:L}\propto\bar{\z}^\infty$.

Recall that in the proof of Theorem~\ref{thm:metathm}, we showed a version of stationarity in the parameter space in eq. \eqref{eq:stationarity-meta}, repeated below.
\begin{equation}
\bar\u^\infty\propto\nabla{\u}\P(\bar{\u}^\infty)\bar\z^\infty. 
\label{eq:p-stat-fcn}
\end{equation} 
%for some positive scalar $D$. 

This case in particular includes $\P_{full}$ which is homogeneous with $\nu=L$. We special case the result fully connected network. In particular, for the parameters of the first layer $\u_1$, we have $\P(\u)=\u_1\u_{2:L}$, where $\u_1\in\bR^{d\times d_{1}}$ and $\u_{2:L}\in\bR^{d_1\times 1}$. 
%$\bar\u_1^\infty$, we have $\P(\bar\u^\infty)=\bar\u_1^\infty \bar\u_{2:L}^\infty$, where $\bar\u_1^\infty\in\bR^{d\times d_{1}}$ and $\bar\u_{2:L}^\infty\in\bR^{d_1\times 1}$. 
This implies, for any $\z$, $ \nabla_{\u_1}\P(\u)\z=\z\u_{2:L}^{\top}$. Using this along with eq. \eqref{eq:p-stat-fcn}, we  get the following expression for some positive scalar $\bar{\gamma}$
\begin{equation}
\bar\u_1^\infty=\bar{\gamma}\,\nabla_{\u_1}\P(\bar\u^\infty)\bar\z^\infty=\bar{\gamma}\,\bar{\z}^\infty\bar\u_{2:L}^{\infty^\top}\implies \bar\u_{1:L}^\infty=\bar\u_1^\infty\bar\u_{2:L}^\infty=\bar{\gamma}\,\bar{\z}^\infty\cdot \norm{\bar\u_{2:L}^\infty}^2\propto \bar{\z}^\infty.
\end{equation}

Since $\bar\u_{1:L}^\infty\propto\tilde\w^\infty$, we have shown that $\tilde\w^\infty\propto \bar{\z}^\infty$, which completes our proof of Theorem~\ref{thm:fcn}. 
\end{proof}

\section{Linear Convolutional  Networks: Proof of Theorem~\ref{thm:conv-net}--\ref{thm:conv-net-l}}\label{app-conv}
Recall that $L$--layer linear convolutional networks have parameters 
$\u=[\u_l\in\bR^{D}]_{l-1}^{L}$. We first recall some complex numbers terminology and properties

\begin{compactenum}
\item Complex vectors $\hat\z\in\b{C}^D$ are represented in polar form as $\hat\z=|\hat\z|\e^{\ci\boldsymbol{\phi}_{\hat\z}}$, where $|\hat\z|\in\bR^D_+$ and $\boldsymbol{\phi}_{\hat\z}\in[0,2\pi)^D$ are the vectors with magnitudes and phases, respectively, of components $\hat\z$.
\item For $\hat\z=|\hat\z|\e^{\ci\boldsymbol{\phi}_{\hat\z}}\in\b{C}^D$, the complex conjugate vector is denoted by $\hat\z^*=|\hat\z|\e^{-\ci\boldsymbol{\phi}_{\hat\z}}$.
\item The complex inner product for $\hat{\x},\hat{\w}\in\b{C}^D$ is given by $\innerprod{\hat\x}{\hat\w}=\sum_d \hat\x{[d]}\hat\w^*[d]=\hat\x^\top \hat{\w}^*$.
\item  Let $\c{F}\in\mathbb{C}^{D\times D}$ denote the discrete Fourier transform matrix with $\c{F}[d,p]=\frac{1}{\sqrt{D}}\omega^{dp}_D$ where recall that $\omega_D=\e^{-\frac{2\pi\ci}{D}}$ is the $D^\text{th}$ complex root of unity. Thus,  for any $\z\in\bR^D$,  the representation in Fourier basis is given by $\hat{\z}=\c{F}\z$. $\c{F}$ and its complex conjugate matrix $\c{F}^*$ also satisfy: $\c{F}\c{F}^*=\c{F}^*\c{F}=I$, $\c{F}=\c{F}^\top$ and $\c{F}^*={\c{F}^*}^\top$.
 \end{compactenum}
 
Before getting into full proofs of Theorem~\ref{thm:conv-net-l}--\ref{thm:conv-net}, we also prove the two lemmas (Lemma~\ref{lem:fft-conv} and Lemma~\ref{lem:fft-gd}) that establish equivalence of dynamics of gradient descent on full dimensional convolutional networks to those on linear diagonal networks (Figure~\ref{fig:dn}), albeit with complex valued parameters. This makes the analysis of the of convolutional networks  simpler and more intuitive. 

%Recall that $\h*\u=\frac{1}{\sqrt{D}}\sum_{k=0}^{D-1}\u[k]\,\h\left[(d+k)\text{ mod } D\right]$. 
We begin by proving Lemma~\ref{lem:fft-conv} which shows the equivalence of representation between convolutional networks and diagonal networks. 
\lemfft*
\begin{proof} 
First, we state the following properties which follow immediately from definitions:
\begin{asparaenum}
\item For $\x,\w\in\bR^D$, 
\begin{equation}
\innerprod{\x}{\w}=\x^\top\w=\x^\top\c{F}\c{F}^*\w=\hat{\x}^\top\hat{\w}^*=\innerprod{\hat\x}{\hat\w}, 
\label{eq:cn-xb}
\end{equation} 
where recall that the complex inner product is given by $\innerprod{\hat\x}{\hat\w}=\hat{\x}^\top\hat{\w}^*$. 
\item 
We next show the following property
\begin{equation} 
\c{F}(\h\star\u)=(\c{F}\h)\odot (\c{F}^*\u)=\hat{\h}\odot \hat\u^*,
\label{eq:cn-conv}
\end{equation}
where $\odot$ denotes the Hadamard product (elementwise product), \ie $(a\odot b)[d]= a[d]b[d]$. 

The above equation follows from simple manipulations of definitions: recall that  $(\c{F}\z)[d]=\frac{1}{\sqrt{D}}\sum_{p=0}^{D-1}\z{[p]}\omega_D^{pd}$ and  $\h\star\u$ defined in eq. \eqref{eq:conv} as $(\h\star\u)[d]=\frac{1}{\sqrt{D}}\sum_{k=0}^{D-1}\u{[k]}\,\h\left[(d+k)\text{ mod } D\right]$.

\begin{align}
\nonumber
\hat{\h}\odot \hat\u^*[d]&=\hat{\h}[d]\hat\u^*[d]=\frac{1}{D}\sum_{k=0}^{D-1}\sum_{k'=0}^{D-1} \u{[k]}\h{[k']}\omega_D^{(k'-k)d}\overset{(a)}=\frac{1}{D}\sum_{k=0}^{D-1}\sum_{k'=0}^{D-1} \u{[k]}\h{[k']}\omega_D^{\left((k'-k)\text{ mod D}\right)d}\\
&\overset{(b)}= \frac{1}{\sqrt{D}}\sum_{p=0}^{D-1}\left[\frac{1}{\sqrt{D}}\sum_{k=0}^{D-1} \u{[k]}\h{[(p+k)\text{ mod }D]}\right]\omega_D^{pd}=\left(\c{F}(\h\star\u)\right)[d],
\end{align}
where $(a)$ follows as $\omega_D^D=1$ and in $(b)$ we used the change of variables $p=(k'-k)\text{ mod }D$ (recall our use of modulo operator as $a\text{ mod }D=a-D\left\lfloor\frac{a}{D}\right\rfloor$).
\end{asparaenum}

Recall from eq. \eqref{eq:cn} the output of an $L$-layer convolutional network is given by \[\hat{y}(\x)=\left((((\x\star\u_1)\star\u_2)\ldots)\star\u_{L-1}\right)^\top \u_L=\innerprod{\x}{\w}.\]  
Denote $\h_{L-1}(\x)=(((\x\star\u_1)\star\u_2)\ldots)\star\u_{L-1}$. By iteratively using eq. \eqref{eq:cn-conv}, we have 
\begin{equation}
\c{F}\h_{L-1}(\x)=\c{F}\x\odot\c{F}^*\u_1\odot\c{F}^*\u_2\ldots\odot\c{F}^*\u_{L-1}.
\label{eq:fhl}
\end{equation}
Thus, on one hand using the above equation we have, 
\begin{equation}
\begin{split}
\hat{y}(\x)&=\h_{L-1}(x)^\top\u_L=\h_{L-1}(\x))^\top \c{F}\c{F}^*\u_L=\left(\c{F}\h_{L-1}(\x)\right)^\top(\c{F}^*\u_L)\\
&\overset{(a)}=\left(\c{F}(\x)\right)^\top\left(\c{F}^*\u_1\odot\c{F}^*\u_2\ldots\odot\c{F}^*\u_{L}\right)\overset{(b)}=\innerprod{\hat\x}{\c{F}\u_1\odot\c{F}\u_2\ldots\odot\c{F}\u_{L}},
\end{split}
\end{equation}
where $(a)$ follows from substituting for $\c{F}\h_{L-1}(\x)$ from eq. \eqref{eq:fhl} and noting that for any $\{\z_l\in\bR^D\}$, $(\z_1\odot\z_2\odot\ldots\z_{L-1})^\top\z_L=\z_1^\top(\z_2\odot\z_3\odot\ldots\z_{L})$, and $(b)$ uses the definition of complex inner product $\innerprod{\hat{\x}}{\hat{\w}}=\hat{\x}^\top\hat{\w}^*$. 

Now further using  eq. \eqref{eq:cn-xb} in above equation, we have 
\begin{equation}
\begin{split}
&\innerprod{\x}{\P_{conv}(\u)}=\innerprod{\hat\x}{\c{F}\P_{conv}(\u)}=\hat{y}(\x)\\
\implies& \innerprod{\hat\x}{\c{F}\P_{conv}(\u)}=\innerprod{\hat\x}{\c{F}\u_1\odot\c{F}\u_2\ldots\odot\c{F}\u_{L}}.
\end{split}
\label{eq:cn-fft-conv}
\end{equation}
Thus, for $\w=\P_{conv}(\u)$, we have shown that $\hat\w=\c{F}\P_{conv}(\u)=\hat\u_1\odot\hat\u_2\ldots\odot\hat\u_{L}=\text{diag}(\hat\u_1)\text{diag}(\hat\u_2)\ldots\text{diag}(\hat\u_{L-1})\hat\u_{L}$. 
\end{proof}
For $\hat\u=[\hat\u_l\in\b{C}^D]_{l=1}^L$, let $\P_{diag}(\hat\u)=\text{diag}(\hat\u_1)\text{diag}(\hat\u_2)\ldots\text{diag}(\hat\u_{L-1})\hat\u_{L}=\hat\u_1\odot\hat\u_2\ldots\odot\hat\u_{L}$ denote the equivalent parameterization of convolutional network in Fourier domain. 

The above lemma shows that optimizing $\c{L}_{\P_{conv}}(\u)$ in eq. \eqref{eq:lmu} is equivalent to the following minimization problem in terms of representation, 
\begin{equation}
\min_{\hat{\u}}\hat{\c{L}}_{\P_{diag}}(\hat\u):=\sum_{n=1}^N\ell(\innerprod{\hat{\x}_n}{\P_{diag}(\hat\u)},y_n)
\label{eq:lmuhat}
\end{equation}

The following lemma further shown that not only the representations of ${\P_{conv}}(\u)$ and $\P_{diag}(\hat\u)$ are equivalent, but there corresponding gradient descent updates for problems in eq. \eqref{eq:lmu} and eq. \eqref{eq:lmuhat} are also equivalent up to Fourier transformations. 
\begin{lemma}\label{lem:fft-gd} Consider the gradient descent iterates $\u[t]=[\u[t]_l]_{l=1}^L$ from eq. \eqref{eq:gd} for minimizing $\c{L}_{\P_{conv}}$ in eq. \eqref{eq:lmu} over full dimensional linear convolutional networks. For all $l$, the incremental update directions,  $\Delta\u[t]_l:={\u[t+1]_l-\u[t]_l}=-\eta_t\nabla_{\u_l}\c{L}_{\P_{conv}}(\u[t])$ satisfy the following, 
\begin{equation}
\c{F}\Delta\u[t]_l={\hat\u^{(t+1)}_l-\hat\u^{(t)}_l}=-\eta_t\nabla_{\hat{\u}_l}\hat{\c{L}}_{\P_{diag}}(\hat\u^{(t)}), 
\end{equation}
where $\hat\u^{(t)}=\left[\hat\u^{(t)}_l\right]_{l=1}^L$ are the Fourier transformations of  $\u[t]=[\u[t]_l]_{l=1}^L$, respectively. 
\end{lemma}
The above lemma shows that Fourier transformation of the gradient descent iterates $\u[t]=[\u[t]_l]_{l=1}^L$ for $\c{L}_{\P_{conv}}$ in eq. \eqref{eq:lmu} are equivalently obtained by gradient descent on the complex parameters $\hat{\u}$ for minimizing $\hat{\c{L}}_{\P_{diag}}$ in eq. \eqref{eq:lmuhat}
\begin{proof} 
We use the notation $\underset{l'\neq l}\odot\hat{\u}_{l'}=\hat\u_1\odot\hat\u_2\ldots \hat\u_{l-1}\odot\hat\u_{l+1}\ldots\odot\hat\u_L$ to denote Hadamard product across all parameters $\hat\u_{l'}$ with  $l'\neq l$.

For any $\u=[\u_l]_{l=1}^L$, using eq. \eqref{eq:cn-fft-conv}, we have the following for all $l$, 
\begin{equation}
\innerprod{\x}{\P_{conv}(\u)}=\innerprod{\hat\x}{\P_{diag}(\hat\u)}=\hat\x^\top\left(\hat{\u}_1^*\odot\hat{\u}_2^*\odot\ldots\hat{\u}_L^*\right)=
\hat{\u}_{l}^{*\top}\Big[\big(\underset{l'\neq l}\odot\hat{\u}^*_{l'}\big)\odot\hat\x\Big].
\end{equation}
Using the above equation we have, 
\begin{equation}
\begin{split}
&\ell(\innerprod{\x}{\P_{conv}(\u)},y_n)=\ell\left({\u}_{l}^\top \;\c{F}^*\Big[\big(\underset{l'\neq l}\odot\hat{\u}^*_{l'}\big)\odot\hat\x\Big],y_n\right)\\
\implies &\nabla_{\u_l}\ell(\innerprod{\x}{\P_{conv}(\u)},y_n)\overset{(a)}=\ell^\prime(\innerprod{\x}{\P_{conv}(\u)},y_n)\c{F}^*\left[\left(\underset{l'\neq l}\odot\hat\u^*_l\right)\odot\hat\x\right]\\
&\quad=\c{F}^*\left[\ell^\prime(\innerprod{\hat\x}{\P_{diag}(\hat\u)},y_n)\left(\underset{l'\neq l}\odot\hat\u^*_l\right)\odot\hat\x\right]
=\c{F}^*\nabla_{\hat\u_l}\ell(\innerprod{\hat\x}{\P_{diag}(\hat\u)},y_n).
\end{split}
\end{equation}
where in $(a)$ we use $\ell'(\hat{y},y)=\pdv{\ell(\hat{y},y)}{\hat{y}}$ and the remaining equalities simply follow from manipulation of derivatives. 
From above equation, we have the following:
\begin{equation}
\begin{split}
\c{F}\Delta\u[t]_l&=-\eta_t\c{F}\nabla_{\u_l}\c{L}_{\P_{conv}}(\u[t])=-\eta_t\c{F}\sum_{n=1}^N\nabla_{\u_l}\ell(\innerprod{\x_n}{\P_{conv}(\u[t])},y_n)\\
&=-\eta_t\c{F}\c{F}^*\sum_{n=1}^N\nabla_{\hat\u_l}\ell(\innerprod{\hat\x_n}{\P_{diag}(\hat{\u}^{(t)})},y_n)=-\eta_t\nabla_{\hat\u_l}\hat{\c{L}}_{\P_{diag}}(\hat{\u}^{(t)}).
%\\
%\ell(\innerprod{\x}{\P_{conv}(\u)})&=\ell(\innerprod{\hat\x}{\P_{diag}(\hat\u)})=\ell\left({\u}_{l}^\top \c{F}^*\Big[\big(\underset{l'\neq l}\odot\hat{\u}^*_{l'}\big)\odot\hat\x\Big]\right)\\
%\implies \nabla_{\u_l}\ell(\innerprod{\x}{\P_{conv}(\u)})&=\c{F}^*\left[\ell^\prime(\innerprod{\x}{\P_{conv}(\u)})\left(\underset{l'\neq l}\odot\hat\u^*_l\right)\odot\hat\x\right].%\\
%&=\c{F}^*\ell^\prime(\innerprod{\x}{\P_{conv}(\u)})\left[\left(\underset{l'\neq l}\odot\c{F}^*\u_l\right)\odot\c{F}\x\right
\qedhere
\end{split}
\end{equation}
\end{proof}

\subsection{Proof of Theorem~\ref{thm:conv-net}--\ref{thm:conv-net-l}}
\thmcn*
\thmcnl*
For the gradient descent iterates $\u[t]=[\u[t]_l]_{l=1}^L$ from eq. \eqref{eq:gd} denote the sequence of corresponding linear predictors as $\w[t]=\P_{conv}(\u[t])$. Let $\hat\w^{(t)}=\c{F}\w[t]$ and $\hat\u^{(t)}_l=\c{F}\u[t]_l$ denote the Fourier transforms of $\w[t]$ and $\u[t]_l$, respectively, and let $\hat\u^{(t)}=\left[\hat\u^{(t)}_l\right]_{l=1}^L$. 

Summarizing the results so far, we have $\hat\w^{(t)}=\hat\u^{(t)}_1\odot\hat\u^{(t)}_2\ldots\odot\hat\u^{(t)}_L$ (from Lemma~\ref{lem:fft-conv}) and $\Delta\hat\u^{(t)}_l:=\hat\u_l^{(t+1)}-\hat\u_l^{(t)}=-\eta_t\nabla_{\hat\u_l}\hat{\c{L}}_{\P_{diag}}(\hat{\u}^{(t)})$ (from Lemma~\ref{lem:fft-gd}).

We use the following observations/notations
\begin{compactenum}
\item 
Let $\bar\u^\infty=\lim\limits_{t\to \infty}\frac{\u[t]}{\norm{\u[t]}}$. 
Denote the Fourier transform of $\bar\u^\infty=[\bar\u_l^\infty]$  as $\hat{\bar\u}^\infty=[\hat{\bar\u}_l^\infty]$. 

Taking Fourier transforms of eq. \eqref{eq:u} which are also applicable here, we have:
\begin{align}
\hat\u^{(t)}_l&=\hat{\bar{\u}}_l^\infty g(t)+\hat{\bdelta}^{(t)}_{\u_l}\,g(t), \label{eq:ucn}
\end{align}
where $g(t)=\norm{\u[t]}=\norm{\hat{\u}^{(t)}}$  and $\hat{\bdelta}^{(t)}_{\u_l}\to 0$. 

\item Denote the negative gradients with respect to $\w[t]$ as $\z[t]=-\nabla_{\w}\c{L}(\w[t])$ and let 
$\hat{\z}^{(t)}=\c{F}\z[t]$. %Since for $\forall n,t$ $\innerprod{\hat{\x}_n}{\hat{\w}^{(t)}}=\innerprod{\x_n}{\w[t]}$, we have the following
%\begin{equation}
%\hat{\z}^{(t)}=\sum_{n}\exp(-\innerprod{{\x}_n}{{\w}^{(t)}})\hat{\x}_n=\c{F}\z[t].
%\end{equation}
 From the assumption of Theorem~\ref{thm:conv-net}-\ref{thm:conv-net-l}, $\lim\limits_{t\to\infty}\frac{\z[t]}{\norm{\z[t]}}$ exists. Let $\bar\z^\infty=\lim_{t\to\infty}\frac{\z[t]}{\norm{\z[t]}}$. Denote  $\hat{\bar\z}^\infty=\c{F}\bar\z^\infty$. We get the following by taking Fourier transform of eq. \eqref{eq:z}
\begin{equation}
\hat{\z}^{(t)}=\hat{\bar\z}^\infty p(t)+\hat{\bdelta}^{(t)}_{\z}\,p(t),
\label{eq:cn-z}
\end{equation}
where $p(t)=\norm{\z[t]}=\norm{\hat{\z}^{(t)}}$ and  $\hat{\bdelta}^{(t)}_{\z}\to0$. 
\item From Lemma~\ref{lem:grad-conv}, we have that $\exists \{\alpha_n\}_{n\in S_\infty}$ such that 
$\lim\limits_{t\to\infty}\frac{\z[t]}{\norm{\z[t]}}=\sum\limits_{n\in S_\infty}\alpha_n\,y_n\x_n$. Thus, 
\begin{equation}
\hat{\bar\z}^\infty=\sum_{n\in S_\infty}\alpha_n\,y_n\hat{\x}_n.
\label{eq:cn-z1}
\end{equation}
\end{compactenum}

\paragraph{KKT conditions for optimality} 
We want to show that a positive scaling of $\bar{\w}^\infty\propto\P_{conv}(\bar{\u}^\infty)$, denoted by $\tilde{\w}^\infty=\gamma\P_{conv}(\bar{\u}^\infty)$ is a first order stationary point of eq. \eqref{eq:gd-cn-l}, repeated below,
\begin{equation*}
\min_{\w}\; \norm{\hat{\w}}_{\nicefrac{2}{L}}\st \forall n,\, y_n\innerprod{\w}{\x_n}\ge1.
\end{equation*}

Recall the KKT conditions discussed in Section~\ref{sec:main}. The first order stationary points, or sub-stationary points, of \eqref{eq:gd-cn-l} are the set of  feasible predictors $\w$ such that $\exists \{\alpha_n\ge 0\}_{n=1}^N$ satisfying the following:  
$\forall n$, $y_n\innerprod{\x_n}{\w}>1\implies \alpha_n=0$, and
 \begin{equation}
 \sum_n\alpha_ny_n\hat{\x}_n\in \partial^{\circ} \norm{\hat{\w}}_p,
 \label{eq:label3-r}
 \end{equation} 
where $\partial^\circ$ denotes the local sub-differential (or Clarke's sub-differential) operator defined as 
$\partial^\circ f(\w)=\text{conv}\{\b{v}:\exists (\z_k)_k\st \z_k\to\w\tand \nabla f(\z_k)\to \b{v}\}.$

For $p=1$ and $\hat{\w}$ represented in polar form as $\hat{\w}=|\hat{\w}|\e^{\ci\boldsymbol{\phi}_{\hat{\w}}}\in\mathbb{C}^D$,  $\norm{\hat{\w}}_p$ is convex and the  local sub-differential is indeed the global sub-differential given by,
  \begin{equation}
\partial^\circ\norm{\hat{\w}}_1=\{\z:\forall d,\; |\z{[d]}|\le 1 \tand \hat{\w}{[d]}\neq0\implies \z{[d]}= \e^{\ci\boldsymbol{\phi}_{\hat{\w}}[d]}\}.
\label{eq:label2-r}
\end{equation} 

For $p<1$, the local sub-differential of  $\norm{\hat{\w}}_p$ is given by, \begin{equation}
\forall p<1,\quad \partial^\circ\norm{\hat{\w}}_p=\{\z:\hat{\w}{[d]}\neq0\implies \z{[d]}=p\; \e^{\ci\boldsymbol{\phi}_{\hat{\w}}[d]}\;|{\hat{\w}}[d]|^{p-1}\}.
\label{eq:label1-r}
\end{equation}

\paragraph{Showing KKT conditions for $\tilde{\w}^\infty\propto\P_{conv}(\bar{\u}^\infty)$.}
As we showed proof of Theorem~\ref{thm:metathm}, since $\P_{conv}(\bar{\u}^\infty)$ has strictly positive margin, using homogeneity of $\P_{conv}$, we can scale $\P_{conv}(\bar{\u}^\infty)$ to get $\tilde{\w}^\infty=\gamma \P_{conv}(\bar{\u}^\infty)$ with unit margin, \ie $\forall n,\,y_n\innerprod{\x_n}{\tilde{\w}^\infty} \ge 1$.
For  dual variables, we again use a positive scaling of $\alpha_n$ from Lemma~\ref{lem:grad-conv}, such that $\bar{\z}^\infty=\sum_{n\in S_\infty}\alpha_n\,y_n\x_n$.

In order to prove the theorem, we need to show that for some positive scalar $\bar{\gamma}$, $ \bar{\gamma}\hat{\bar{\z}}^\infty\in \partial^\circ\norm{\hat{\w}}_{\nicefrac{2}{L}}$, \ie satisfies the conditions in eq. \eqref{eq:label2-r} and \eqref{eq:label1-r}, for $L=2$ and $L>2$, respectively.  

We start from the stationarity condition in the parameter space in eq. \eqref{eq:stationarity-meta} of Theorem~\ref{thm:metathm}. For some positive scalar $\bar{\gamma}$, we have
\begin{equation}
\bar\u^\infty=\bar{\gamma}\nabla{\u}\P_{conv}(\bar{\u}^\infty)\bar\z^\infty. 
\label{eq:c-p-stat}
\end{equation} 
We will now special case the above equation for fully width convolutional networks. 

From Lemma~\ref{lem:fft-conv}, we have that for all $\u=[\u_{l}\in\bR^D]$, we have $\P_{conv}(\u)=\c{F}^*\P_{diag}(\c{F}\u)$ where $\c{F}$ and $\c{F}^*$ denote discrete Fourier matrix and its inverse in appropriate dimensions. %In just the following set of equations, all the Fourier matrices are of dimension $D\times D$. 
Let $\{e_d\}_{d=1}^D$ denote the standard basis in $\bR^D$. We first note that for all $l=1,2,\ldots,L$ and for all $d=1,2,\ldots,D$, the following holds
\begin{align}
\P_{conv}(\u)[d]&=e_d^\top\c{F}^*\P_{diag}(\c{F}\u)=
e_d^\top\c{F}^*\left(\odot_{l'=1}^{L-1}\hat{\u}_{l'}\right)\\
&=e_d^\top\c{F}^*\left(\prod_{l'\neq l}\diag(\hat{\u}_{l'})\right)\c{F}{\u}_l=\innerprod{\u_l}{\c{F}^*\left(\prod_{l'\neq l}\diag(\hat{\u}^*_{l'})\right)\c{F}e_d}.\\
\implies& \nabla_{\u_l}\P_{conv}(\u)[:,d]=\c{F}^*\left(\prod_{l'\neq l}\diag(\hat{\u}^*_{l'})\right)\c{F}e_d. 
\end{align}
This implies, for $l=1,2,\ldots,L$ and any $\z\in \bR^D$, we have 
\begin{align}
\nabla_{\u_{l}}\P_{conv}(\u)\z=\sum_{d}\nabla_{\u_l}\P_{conv}(\u)[:,d]{\z}[d]=\c{F}^*\left(\prod_{l'\neq l}\diag(\hat{\u}^*_{l'})\right)\c{F}\z.
\end{align}

Substituting the above equation in eq. \eqref{eq:c-p-stat}, we have,
\begin{equation}
\begin{split}
\hat{\bar{\u}}_l^\infty&=\c{F}\bar{\u}_l^\infty=\bar{\gamma}\c{F}\nabla_{\u_l}\P_{conv}(\bar{\u}^\infty)\bar{\z}^\infty=\bar{\gamma}\left(\odot_{l'\ne l}\hat{\bar{\u}}_{l'}^{\infty*}\right)\odot \hat{\bar\z}^\infty,
\end{split}
\label{eq:cn-slackness-p}
\end{equation}
where $\hat{\bar\u}^{\infty*}_{l'}$ denotes the complex conjugate of $\hat{\bar\u}^{\infty}_{l'}$.

Let $\hat{\bar\w}^\infty=\P_{diag}(\hat{\bar{\u}}^\infty)$. The above equation, further implies, for all $l$
\begin{equation}
|\hat{\bar{\u}}_l^{\infty}|^2=\hat{\bar{\u}}_l^{\infty*}\odot\hat{\bar{\u}}_l^\infty=\bar{\gamma}\,\hat{\bar\w}^{\infty*}\odot \hat{\bar\z}^\infty=\bar{\gamma}|\hat{\bar\w}^\infty|\odot|\hat{\bar\z}^\infty|\e^{\ci(\phi_{\hat{\bar\z}^\infty}-\phi_{\hat{\bar\w}^\infty})}
\label{eq:cn_slack_pnorm}
\end{equation}

In eq. \eqref{eq:cn_slack_pnorm}, since the LHS is a real number, we have that for all $d$ such that $|\hat{\bar\w}^\infty[d]|>0$
\begin{equation}
\e^{\ci\phi_{\hat{\bar\z}^\infty[d]}}=\e^{\ci\phi_{\hat{\bar\w}^\infty [d]}}. 
\label{eq:cn-slackness-phase}
\end{equation}

Also, by multiplying the LHS of eq. \eqref{eq:cn_slack_pnorm} across all $l$ and taking $L$th root over positive scalars, we have for $d=0,1,\ldots,D-1$,
\begin{equation}
\left|\hat{\bar{\w}}^{\infty}[d]\right|^{\nicefrac{2}{L}}=\bar{\gamma}\left|\hat{\bar\w}^\infty[d]\right|\;|\hat{\bar\z}^\infty[d]|,
\label{eq:cn-slackness-mag}
\end{equation}

Finally, let $\gamma$ be a positive scaling of $\bar\w^\infty$ such that $\tilde{\w}^\infty=\gamma\bar\w^\infty$ has unit margin. Let $\hat{\tilde\w}^\infty=\c{F}\tilde{\w}^\infty=\gamma \hat{\bar\w}^\infty$. Since $\bar{\gamma}$ is arbitrary positive scalar, redefining as $\bar{\gamma}\gets\frac{2}{L}\gamma^{\nicefrac{2}{L}-1}\bar{\gamma}$, we have from eq. \eqref{eq:cn-slackness-phase}-\eqref{eq:cn-slackness-mag},
\begin{equation}
\forall d\st \Big|\hat{\tilde{\w}}^{\infty}[d]\Big|\neq0,\quad 
\bar{\gamma}\;\hat{\bar\z}[d]=\e^{\ci\phi_{\hat{\bar\w}}[d]} \left|\hat{\tilde{\w}}^{\infty}[d]\right|^{\nicefrac{2}{L}-1}
\label{eq:cn-slackness-all}
\end{equation}

\subsubsection{Case of $L>2$ or $p=\nicefrac{2}{L}<1$}
For $p=\nicefrac{2}{L}<1$, since $\hat{\bar\z}^\infty=\sum_{n\in S_\infty}\alpha_ny_n\hat\x_n$,  eq. \eqref{eq:cn-slackness-all} is indeed the first order stationarity condition for eq. \eqref{eq:gd-cn-l} as described in eq. \eqref{eq:label3} and \eqref{eq:label1}. 
\subsubsection{Case of $L=2$ or $p=\nicefrac{2}{L}=1$}

For the case of $p=1$, in addition to eq. \eqref{eq:cn-slackness-all}, we need to show that $\bar{\gamma}|\hat{\bar{\z}}^\infty|\le 1$. From  eq. \eqref{eq:cn-slackness-all}, for $L=2$ we have $\Big|\hat{\tilde{\w}}^{\infty}[d]\Big|\neq0\implies \bar{\gamma} |\hat{\bar{\z}}^\infty[d]|=1$. 

We need to further show that $\forall d\st \Big|\hat{\tilde{\w}}^{\infty}[d]\Big|\propto \Big|\hat{\bar{\w}}^{\infty}[d]\Big|=0$, $\bar{\gamma}|\hat{\bar{\z}}^\infty[d]|\le 1$. 
\subsubsection*{Showing  $\forall d, \Big|\hat{\bar{\w}}^{\infty}[d]\Big|=0\implies\bar{\gamma}|\hat{\bar{\z}}^\infty[d]|\le 1$}
Using Lemma~\ref{lem:fft-gd} for for the special case of $2$--layer linear convolutional network, for $\forall d$,
\begin{equation}
\begin{split}
\Delta\hat{\u}_1^{(t)}[d]&=\eta_t\hat{\z}^{(t)}[d]\,{{\hat{\u}^{(t)*}_{2}}}[d],\\
\Delta\hat{\u}_2^{(t)}[d]&=\eta_t\hat{\z}^{(t)}[d]\,{{\hat{\u}^{(t)*}_{1}}}[d].
\end{split}
\label{eq:cn-grad-2}
\end{equation}

\remove{
\paragraph{Simple case: real valued diagonal networks} For simplicity, we first prove this part of the result for real valued diagonal networks. The proof for complex variables is analogous. The corresponding update equations in eq. \eqref{eq:cn-grad-2} for real valued diagonal networks are given below:
\begin{equation}
\begin{split}
\Delta{\u}_1^{(t)}[d]&=\eta_t{\z}^{(t)}[d]\,{{{\u}^{(t)}_{2}}}[d],\\
\Delta{\u}_2^{(t)}[d]&=\eta_t{\z}^{(t)}[d]\,{{{\u}^{(t)}_{1}}}[d].
\end{split}
\label{eq:diag-grad-2}
\end{equation}
Now consider the dynamics of $\b{u}_d^{(t)}:=\text{sign}(\bar\u_1^\infty){\u}_1^{(t)}[d]+\text{sign}(\bar\u_2^\infty){\u}_2^{(t)}[d]$.

Since for $l=1,2$, $\u[t]_l/g(t)\to\bar{\u}_l^\infty$, we have the following:
\begin{equation}
\begin{split}
\lim_{t\to\infty}\frac{\b{u}_d^{(t)}}{g(t)}&=\text{sign}(\bar\u_1^\infty){\bar\u}_1^\infty[d]+\text{sign}(\bar\u_2^\infty){\bar\u}_2^{\infty}[d]=|{\bar\u}_1^\infty[d]|+|{\bar\u}_2^{\infty}[d]|
%\overset{(a)}=\text{sign}(\bar\u_1^\infty[d])\left[|\bar\u_1^\infty[d]|+|\bar\u_2^\infty[d]|\right]\\
\overset{(a)}=2\sqrt{|\bar\w^\infty[d]|}
\end{split}
\label{eq:cn-uddiag}
\end{equation}
where $(a)$ follows from eq. \eqref{eq:cn_slack_pnorm} which implies that for all $l$, $|\bar\u_l^\infty[d]|$ have the same magnitude, and hence  $|\bar\u_l^\infty[d]|=|\bar\w^\infty[d]|^{1/L}$. 

Now looking at the dynamics of $\b{u}_d$, using eq. \eqref{eq:diag-grad-2} we have that 
\begin{equation}
\begin{split}
\b{u}_d^{(t+1)}&=\b{u}_d^{(t)}+\text{sign}(\bar\u_1^\infty)\eta_t\z[t][d]{\u}_2^{(t)}[d]+\text{sign}(\bar\u_2^\infty)\eta_t\z[t][d]{\u}_1^{(t)}[d]\\
&=\left(1+\text{sign}(\bar\u_1^\infty)\text{sign}(\bar\u_2^\infty)\z[t][d]\right)\b{u}_d^{(t)}\overset{(a)}=
(1+\text{sign}(\bar\z^\infty)\z[t][d])\b{u}_d^{(t)}\\
&\overset{(b)}=\left(1+\eta_t\left(|\bar{\z}^\infty[d]|+\bdelta[t][d]\right)p(t)\right)\b{u}_d^{(t)}
\end{split}
\end{equation}
where $(a)$ uses $\text{sign}(\bar\z^\infty)=\text{sign}(\bar\w^\infty)=\text{sign}(\bar\u_1^\infty)\text{sign}(\bar\u_2^\infty)$  from eq. \eqref{eq:cn-slackness-phase},  and $(b)$ follows from using $\z[t]=\bar{\z}^\infty p(t)+\bdelta[t]_zp(t)$ for some $\bdelta[t]_z\to0$ (from convergence in direction of $\z[t]$).

We prove our theorem by looking the following new quantity: For any $d,d'$, denote $\kappa_{d,d'}^{(t)}=\left|\frac{\b{u}_d^{(t)}}{\b{u}_{d'}^{(t)}}\right|$. 
 In the remainder of the proof, we will show that  $|\bar{\z}^\infty[d]|>|\bar{\z}^\infty[d']|\implies \kappa_{d,d'}^{(t)}\to\infty$. 
Along with eq. \eqref{eq:cn-ud}, this would imply  that  $\lim\limits_{t\to\infty}\kappa_{d,d'}^{(t)}=\sqrt{\frac{|\bar\w^\infty[d]|}{|\bar\w^\infty[d']|}}=\infty$. 
Hence, for any $d,d'$ with $\bar\w^\infty[d]=0$ and $\bar\w^\infty[d']\neq0$, we have $\bar{\gamma}|\bar\z^\infty[d]|\le\bar{\gamma}|\bar\z^\infty[d']|=1$ (last equality follows from eq.\eqref{eq:cn-slackness-mag}). 

\textit{Showing $|\bar{\z}^\infty[d]|>|\bar{\z}^\infty[d']|\implies \kappa_{d,d'}^{(t)}\to\infty$:} For any $2\epsilon>0$, let $|\bar{\z}^\infty[d]|-|\bar{\z}^\infty[d']|=2\epsilon>0$. We note that the since the loss $\c{L}(\w)$ is minimized, norm of the gradient $p(t)=\norm{\z[t]}\to0$. Hence, for any finite step size sequence $\{\eta_t\}$ and hence for large enough $t_1$, $\forall t\ge t_1$ and $\forall d$, $1+\eta_t\left(\left|\bar{\z}^\infty[d]\right|p(t)+\bdelta[t][d]p(t)\right)\in[0.5,1.5]$ and the following inequalities hold,
\begin{align}
\kappa_{d,d'}^{(t+1)}&=\left|\frac{\b{u}_d^{(t+1)}}{\b{u}_{d'}^{(t+1)}}\right|=\left|\frac{\left(1+\eta_t\left(|\bar{\z}^\infty[d]|+\bdelta[t][d]\right)p(t)\right)}{\left(1+\eta_t\left(|\bar{\z}^\infty[d']|+\bdelta[t][d']\right)p(t)\right)} \frac{\b{u}_d^{(t)}}{\b{u}_{d'}^{(t)}}\right|\\
&\overset{(a)}=\frac{\left(1+\eta_t\left(|\bar{\z}^\infty[d]|+\bdelta[t][d]\right)p(t)\right)}{\left(1+\eta_t\left(|\bar{\z}^\infty[d']|+\bdelta[t][d']\right)p(t)\right)} \kappa_{d,d'}^{(t)}\\
&\overset{(b)}\ge \left(1+\eta_t\left(|\bar{\z}^\infty[d]|-|\bdelta[t][d]|\right)p(t)\right)\left(1-\eta_t\left(|\bar{\z}^\infty[d']|+|\bdelta[t][d']|\right)p(t)\right)\kappa_{d,d'}^{(t)}\\
&\overset{(c)}\ge \left(1+\eta_t\left(2\epsilon+\delta(t)\right)p(t)\right)\kappa_{d,d'}^{(t)},
\end{align}
where in $(a)$ follows since  $1+\eta_t\left(|\bar{\z}^\infty[d]|p(t)+\bdelta[t][d]p(t)\right)>0.5$,  $(b)$ follows from using $\nicefrac{1}{(1+x)}\ge (1-x)$ for $x< 1$, and finally in $(c)$, we absorbed all $o(p(t))$ terms as $\delta(t)p(t)$ for $\delta(t)\to0$ and used $|\bar{\z}^\infty[d]|-|\bar{\z}^\infty[d']|=2\epsilon>0$. 

Since $\delta(t)\to0$, for large enough $t_2$ and $t\ge t_2$, we have $\delta(t)<\epsilon$. Thus, for all $t\ge \max\{t_1,t_2\}$, 
\begin{equation}
\kappa_{d,d'}^{(t+1)}\ge \left(1+\eta_t\epsilon p(t)\right)\kappa_{d,d'}^{(t)}.
\label{eq:kt}
\end{equation}

Further,  from the conditions of the theorem, we have the following for $l=1,2$:
\begin{asparaenum}
\item For almost all initializations, $|\u[0]_l[d]|>0$ for all $d$, and further, for step sizes $\{\eta_t\}$ smaller than the local Lipschitz constant, $\forall t^\prime<\infty$ and $\forall d$, we also have $|\u[t^\prime]_l[d]|>0$. 
\item Since $\u[t]_l/g(t)\to\bar{\u}_l^\infty$, $\exists t_3$ such that $\forall t\ge t_3$ and $\forall d$, $\text{sign}(\u[t]_l[d])=\text{sign}(\bar{\u}_l^\infty [d])$. 
\end{asparaenum}

Thus, for  $t_0=\max\{t_1,t_2,t_3\}$, using the above observations, we have that $\kappa_{d,d'}^{(t_0)}=\frac{\text{sign}(\bar\u_1^\infty){\u}_1^{(t_0)}[d]+\text{sign}(\bar\u_2^\infty){\u}_2^{(t_0)}[d]}{\text{sign}(\bar\u_1^\infty){\u}_1^{(t_0)}[d']+\text{sign}(\bar\u_2^\infty){\u}_2^{(t_0)}[d']}=\frac{|{\u}_1^{(t_0)}[d]|+|{\u}_2^{(t_0)}[d]|}{|{\u}_1^{(t_0)}[d']|+|{\u}_2^{(t_0)}[d']|}>0$. 
Now, using eq. \eqref{eq:kt}, for all $t\ge t_0$, 
\begin{align}
\kappa_{d,d'}^{(t+1)}\ge 
(1+\eta_t\epsilon p(t))\kappa_{d,d'}^{(t)}=\left(\prod_{u=t_0}^t(1+\eta_u\epsilon p(u))\right)\kappa_{d,d'}^{(t_0)} \tand \kappa_{d,d'}^{(t_0)}>0.
\label{eq:prodk}
\end{align}

Finally, we show the following claim:
\begin{claim} For any finite $t_0$, finite step-sizes $\{\eta_t\}$, and any $\epsilon>0$, we have $\prod_{u=t_0}^t(1+\eta_u\epsilon p(u))\to\infty$.
\end{claim}
\begin{proof} Let $\mu=\max_{d} |\bar\z^\infty[d]|+\max_{t>t_0}|\bdelta[t][d]|<\infty$. 
From the definition, we have that  \[|\b{u}^{(t+1)}_d|\le(1+\mu\eta_t p(t))|\b{u}^{(t)}_d|\le |\b{u}^{(t_0)}_d|\prod_{u=t_0}^t(1+\mu \eta_up(u))\le |\b{u}^{(t_0)}_d|\exp(\sum_{u=t_0}^t\mu\eta_u p(u)).\] 
Moreover, we have $\b{u}^{(t)}_d\to\infty$ for at least one $d$, and for any finite step sizes and  finite $t_0$, $|\b{u}^{(t_0)}_d|<\infty$. This then implies that for some  $\mu<\infty$, $\exp(\sum_{u=t_0}^t\mu\eta_u p(u))\to\infty\implies \sum_{u=t_0}^t \eta_up(u)\to\infty$. 
Thus, for any $\epsilon>0$, we also have $\prod_{u=t_0}^t(1+\epsilon \eta_up(u))\ge \epsilon\sum_{u=t_0}^t \eta_u p(u)\to\infty$.
\end{proof}

%Recall from eq. \eqref{eq:cn-ud} that $\kappa_{d,d'}^{(t)}=\left|\frac{\b{u}_d^{(t)}}{\b{u}_{d'}^{(t)}}\right|=\sqrt{\frac{|\w[t][d]|}{|\w[t][d']|}}$. 

From eq. \eqref{eq:prodk} and above claim, we conclude that for all $d,d'$, if $|\bar{\z}^\infty[d]|>|\bar{\z}^\infty[d']|$, then $\kappa_{d,d'}^{(t)}\to\infty$. 
This completes the proof real valued diagonal networks. 
The complex valued networks (convolutions) can be handled in an analogous way as below.

\subsection{Complex valued diagonal networks (convolutional networks)}}
Recall: for $l=1,2$, $\frac{\hat{\u}_l^{(t)}}{g(t)}\to\hat{\bar{\u}}_l^\infty$, $\frac{\hat{\z}^{(t)}}{p(t)}\to\hat{\bar{\z}}^\infty$, $\hat{\w}^{(t)}=\hat{\u}_1^{(t)}\odot \hat{\u}_2^{(t)}$ and $\frac{\hat{\w}^{(t)}}{g(t)^2}\to\hat{\bar\w}^\infty=\hat{\bar\u}_1^{\infty}\odot \hat{\bar\u}_2^{\infty}$. 

Further, from  eq. \eqref{eq:cn_slack_pnorm}, we have  $\forall d$, $|\hat{\bar\u}_1^\infty[d]|^2=|\hat{\bar\u}_2^\infty[d]|^2$, and hence  
\begin{equation}
|\hat{\bar\u}_1^\infty[d]|=|\hat{\bar\u}_2^\infty[d]|=\sqrt{|\hat{\bar\w}^\infty[d]|}.
\label{eq:cnt1}
\end{equation}

From the convergence of complex numbers, we have the following:
\begin{compactenum}
\item $\forall d$ such that $|\hat{\bar\z}^\infty[d]|\neq0$, we have
\begin{equation}
\frac{|\hat{\z}^{(t)}[d]|}{p(t)}\to|\hat{\bar\z}^\infty[d]| \text{ and } \e^{\ci\phi_{\hat{\z}^{(t)}[d]}}\to\e^{\ci\phi_{\hat{\bar\z}^\infty[d]}}.
\end{equation}
\item $\forall d$ such that $|\hat{\bar\w}^\infty[d]|\neq0$, we have $|\hat{\bar\u}_1^\infty[d]|,|\hat{\bar\u}_2^\infty[d]|\neq 0$, and the following holds
\begin{equation}
\begin{split}
\text{for }l=1,2,\quad&\frac{|\hat{\u}^{(t)}_l[d]|}{g(t)}\to|\hat{\bar\u}_l^\infty[d]| \;\text{ and }\; \e^{\ci\phi_{\hat{\u}^{(t)}_l[d]}}\to \e^{\ci\phi_{\hat{\bar\u}_l^\infty[d]}}\\
&\frac{|\hat{\w}^{(t)}[d]|}{g(t)^2}\to|\hat{\bar\w}^\infty[d]| \;\text{ and }\; \e^{\ci\phi_{\hat{\w}^{(t)}[d]}}\to \e^{\ci\phi_{\hat{\bar\w}^\infty[d]}}=\e^{\ci\phi_{\hat{\bar\u}_1^{\infty}[d]}}\cdot\e^{\ci\phi_{\hat{\bar\u}_2^{\infty}[d]}}\\
&\bar{\gamma}|\hat{\bar\z}^\infty[d]|=1\;\text{ and }\;\e^{\ci\phi_{\hat{\bar\z}^\infty[d]}}=\e^{\ci\phi_{\hat{\bar\w}^\infty[d]}},
\end{split}
\label{eq:cnt2}
\end{equation}
where the last equation follows from eq. \eqref{eq:cn-slackness-phase}.
\item $\forall d$ such that $|\hat{\bar\w}^\infty[d]|=0$, from eq. \eqref{eq:cnt1}, we have $|\hat{\bar\u}_1^\infty[d]|=|\hat{\bar\u}_2^\infty[d]|=0$. 
\end{compactenum}

In the remainder of the proof, we only consider $d$ with $|\hat{\bar\z}^\infty[d]|\neq0$. 

Consider $\b{u}_d^{(t)}$ defined below,
\begin{equation}
\b{u}_d^{(t)}:=\hat{\u}_1^{(t)}[d]\cdot\e^{-\ci\phi_{\hat{\bar\z}^\infty[d]}}+\hat{\u}_2^{(t)*}[d].
\label{eq:cn-ud}
\end{equation}

Since for $l=1,2$, $\u[t]_l/g(t)\to\bar{\u}_l^\infty$, we have the following:
\begin{align}
\nonumber \lim_{t\to\infty}\frac{\b{u}_d^{(t)}}{g(t)}&=\hat{\bar\u}_1^\infty[d]\cdot \e^{-\ci\phi_{\hat{\bar\z}^\infty[d]}}+\hat{\bar\u}_2^{\infty*}[d]
\overset{(a)}=\left\{\begin{array}{ll}
0&\text{if }|\hat{\bar\w}^\infty[d]|=0\\
\e^{-\ci\phi_{\hat{\bar\u}_2^\infty[d]}}\left[|\hat{\bar\u}_1^\infty[d]|+|\hat{\bar\u}_2^\infty[d]|\right]&\text{if }|\hat{\bar\w}^\infty[d]|>0
\end{array}\right.\\
&\overset{(b)}=\left\{\begin{array}{ll}
0&\text{if }|\hat{\bar\w}^\infty[d]|=0\\
2\e^{-\ci\phi_{\hat{\bar\u}_2^\infty[d]}}\sqrt{|\hat{\bar\w}^\infty[d]|}&\text{if }|\hat{\bar\w}^\infty[d]|>0
\end{array}\right.,
\label{eq:cn-udlim}
\end{align}
where $(a)$ follows from using $\e^{\ci\phi_{\hat{\bar\z}^\infty[d]}}=\e^{\ci\phi_{\hat{\bar\w}^\infty[d]}}=\e^{\ci\phi_{\hat{\bar\u}_1^\infty[d]}}\cdot \e^{\ci\phi_{\hat{\bar\u}_2^\infty[d]}}$ whenever $\bar\w^\infty[d]\neq0$ (from eq. \eqref{eq:cnt2}), and $(b)$ follows from  eq. \eqref{eq:cnt1}. 

\begin{asparaenum}[Step 1.]
\item \textit{Dynamics of $\b{u}_d^{(t)}$: }
Now looking at the dynamics of $\b{u}_d$, using eq. \eqref{eq:cn-grad-2} we have that 
\begin{equation*}
\begin{split}
\b{u}_d^{(t+1)}&=\b{u}_d^{(t)}+\e^{-\ci\phi_{\hat{\bar\z}^\infty[d]}}\cdot\eta_t\hat{\z}^{(t)}[d]\,{\hat\u}_2^{(t)*}[d]+\eta_t\hat{\z}^{(t)*}[d]\,{\hat\u}_1^{(t)}[d]\\
&=\b{u}_d^{(t)}+\eta_t|\hat{\z}^{(t)}[d]|\left[\e^{\ci\left(\phi_{\hat{\z}^{(t)}[d]}-\phi_{\hat{\bar\z}^\infty[d]}\right)}
{\hat\u}_2^{(t)*}[d]+\e^{-\ci\left(\phi_{\hat{\z}^{(t)}[d]}-\phi_{\hat{\bar\z}^\infty[d]}\right)}\cdot{\hat\u}_1^{(t)}[d]\cdot\e^{-\ci\phi_{\hat{\bar\z}^\infty[d]}}\right]%\\
%&\overset{(b)}=\left(1+\eta_t\left(|\bar{\z}^\infty[d]|+\bdelta[t][d]\right)p(t)\right)\b{u}_d^{(t)}
\end{split}
\end{equation*}

Additionally, since $\e^{\ci\phi_{\hat{\z}^{(t)}[d]}}\to\e^{\ci\phi_{\hat{\bar\z}^\infty[d]}}$, we can write $\e^{\pm\ci\left(\phi_{\hat{\z}^{(t)}[d]}-\phi_{\hat{\bar\z}^\infty[d]}\right)}=1+\bdelta^{(t)}_{1,d}\pm\ci\bdelta^{(t)}_{2,d}$ where $\bdelta^{(t)}_{1,d},\bdelta^{(t)}_{2,d}\to0$ are real scalars. Substituting in above equation and rearranging the terms, we have
\begin{align}
\nonumber \b{u}_d^{(t+1)}&=\left[1+\eta_t|\hat{\z}^{(t)}[d]|(1+\bdelta^{(t)}_{1,d})\right]\b{u}_d^{(t)}+\ci\bdelta^{(t)}_{2,d}\eta_t|\hat{\z}^{(t)}[d]|\left[
{\hat\u}_2^{(t)*}[d]-{\hat\u}_1^{(t)}[d]\cdot\e^{-\ci\phi_{\hat{\bar\z}^\infty[d]}}\right]\\
&\overset{(a)}{:=}\left[1+\eta_t|\hat{\z}^{(t)}[d]|(1+\bdelta^{(t)}_{1,d})\right]\b{u}_d^{(t)}+\eta_t|\hat{\z}^{(t)}[d]|\boldsymbol{\tau}_d^{(t)},
%&\overset{(b)}=\left(1+\eta_t\left(|\bar{\z}^\infty[d]|+\bdelta[t][d]\right)p(t)\right)\b{u}_d^{(t)}
\label{eq:cn-ud+}
\end{align}
where in $(a)$ we define $\boldsymbol{\tau}_d^{(t)}=\ci\bdelta^{(t)}_{2,d}\left[
{\hat\u}_2^{(t)*}[d]-{\hat\u}_1^{(t)}[d]\cdot\e^{-\ci\phi_{\hat{\bar\z}^\infty[d]}}\right]$.

The following intermediate lemma is proved in Appendix~\ref{app:cnlem}.
\begin{restatable}{lemma}{lemcnlemma} \label{lem:cn-lemma}
Consider $\boldsymbol{\tau}_d^{(t)}$ in eq. \eqref{eq:cn-ud+}. For all $d$ such that $\hat{\bar\z}^\infty[d]\neq 0$, $\b{u}_d^{(t)}\to\infty$ and $\frac{\boldsymbol{\tau}_d^{(t)}}{\b{u}_d^{(t)}}\to0$.
\end{restatable}
Using the above lemma, we have $\bdelta^{(t)}_{3,d}\to0$ such that $\boldsymbol{\tau}_d^{(t)}=\bdelta^{(t)}_{3,d}\b{u}_d(t)$. Additionally, since $\frac{|\hat{\z}^{(t)}[d]|}{p(t)}\to|\hat{\bar\z}^\infty[d]|$, there exists $\bdelta^{(t)}_{4,d}\to0$ such that $|\hat{\z}^{(t)}[d]|=|\hat{\bar\z}^\infty[d]|p(t)+\bdelta^{(t)}_{4,d} p(t)$. Substituting these representations in  eq. \eqref{eq:cn-ud+}, we have the following dynamics for $\b{u}_d(t)$,

\begin{equation}
\begin{split}
\b{u}_d^{(t+1)}&=\left[1+\eta_t p(t)\left(|\hat{\bar\z}^\infty[d]|+\bdelta^{(t)}_{4,d}\right)\left(1+\bdelta^{(t)}_{1,d}+\bdelta^{(t)}_{3,d}\right)\right]\b{u}_d^{(t)}\\
&\overset{(a)}{:=}\left[1+\eta_t p(t)\left(|\hat{\bar\z}^\infty[d]|+\bdelta^{(t)}_{d}\right)\right]\b{u}_d^{(t)},
\end{split}
\label{eq:cn-ud++}
\end{equation}
where in $(a)$ we have accumulated all diminishing terms into $\bdelta^{(t)}_{d}=\bdelta^{(t)}_{4,d}\left(1+\bdelta^{(t)}_{1,d}+\bdelta^{(t)}_{3,d}\right)+|\hat{\bar\z}^\infty[d]|\left(\bdelta^{(t)}_{1,d}+\bdelta^{(t)}_{3,d}\right)\to0$. 

\item \textit{Remainder of the proof: }We now prove our theorem by looking the following quantity: For any $d,d'$ with $\hat{\bar\z}^\infty[d],\hat{\bar\z}^\infty[d']\neq 0$, define $\kappa_{d,d'}^{(t)}=\left|\frac{\b{u}_d^{(t)}}{\b{u}_{d'}^{(t)}}\right|$. 

We will show that  whenever $|\hat{\bar{\z}}^\infty[d]|>|\hat{\bar{\z}}^\infty[d']|$, we  get $\kappa_{d,d'}^{(t)}\to\infty$. 
Along with eq. \eqref{eq:cn-udlim}, this would imply  that  $\lim\limits_{t\to\infty}\kappa_{d,d'}^{(t)}=\sqrt{\frac{|\hat{\bar{\w}}^\infty[d]|}{|\hat{\bar{\w}}^\infty[d']|}}=\infty$. 
Hence, for any $d,d'$ with $\hat{\bar{\w}}^\infty[d]=0$ and $\hat{\bar{\w}}^\infty[d']\neq0$, we have $\bar{\gamma}|\hat{\bar\z}^\infty[d]|\le\bar{\gamma}|\hat{\bar\z}^\infty[d']|$. Moreover from eq.\eqref{eq:cn-slackness-mag}), we  know that $\bar{\gamma}|\hat{\bar{\z}}^\infty[d']|=1$ for all $d'$ with $\hat{\bar{\w}}^\infty[d']\neq0$. This implies $\forall d$, $\bar{\gamma}|\hat{\bar\z}^\infty[d]|\le1$ and concludes the proof.

\textit{Showing $|\hat{\bar{\z}}^\infty[d]|>|\hat{\bar{\z}}^\infty[d']|\implies \kappa_{d,d'}^{(t)}\to\infty$:} 

For any $2\epsilon>0$, let $|\hat{\bar{\z}}^\infty[d]|-|\hat{\bar{\z}}^\infty[d']|=2\epsilon>0$. We note that the since the loss $\c{L}(\w[t])\to0$, norm of the gradient $p(t)=\norm{\z[t]}=\norm{\hat{\z}^{t}}\to0$. Hence, for any finite step size sequence $\{\eta_t\}$, there exists $t_1$ such that $\forall t\ge t_1$ and $\forall d$, $\eta_tp(t)\left(|\hat{\bar{\z}}^\infty[d]|+|\bdelta[t]_d|\right)< 0.5$ and the following inequalities hold,
\begin{align}
\kappa_{d,d'}^{(t+1)}&=\left|\frac{\b{u}_d^{(t+1)}}{\b{u}_{d'}^{(t+1)}}\right|=\left|\frac{\left(1+\eta_t\left(|\hat{\bar{\z}}^\infty[d]|+\bdelta[t]_d\right)p(t)\right)}{\left(1+\eta_t\left(|\hat{\bar{\z}}^\infty[d']|+\bdelta[t]_d\right)p(t)\right)} \frac{\b{u}_d^{(t)}}{\b{u}_{d'}^{(t)}}\right|\\
&\ge \frac{\left(1+\eta_t\left(|\hat{\bar{\z}}^\infty[d]|-|\bdelta[t]_d|\right)p(t)\right)}{\left(1+\eta_t\left(|\hat{\bar{\z}}^\infty[d']|+|\bdelta[t]_{d'}|\right)p(t)\right)} \kappa_{d,d'}^{(t)}\\
&\overset{(a)}\ge \left(1+\eta_t\left(|\hat{\bar{\z}}^\infty[d]|-|\bdelta[t]_d|\right)p(t)\right)\left(1-\eta_t\left(|\hat{\bar{\z}}^\infty[d']|+|\bdelta[t]_{d'}|\right)p(t)\right)\kappa_{d,d'}^{(t)}\\
&\overset{(c)}\ge \left(1+\eta_t\left(2\epsilon+\bdelta^{(t)}_{d,d'}\right)p(t)\right)\kappa_{d,d'}^{(t)},
\end{align}
where in $(a)$ follows from using $\nicefrac{1}{(1+x)}\ge (1-x)$ for $x< 1$ since $\eta_tp(t)\left(|\hat{\bar{\z}}^\infty[d]|+|\bdelta[t]_d|\right)< 0.5$ for all $t\ge t_1$, and  in $(c)$, we absorbed all $o(p(t))$ terms as $\bdelta^{(t)}_{d,d'}p(t)$ for $\bdelta^{(t)}_{d,d'}\to0$ and used $|\hat{\bar{\z}}^\infty[d]|-|\hat{\bar{\z}}^\infty[d']|=2\epsilon>0$. 

Since $\bdelta^{(t)}_{d,d'}\to0$, for large enough $t_2$ and $t\ge t_2$, we have $|\bdelta^{(t)}_{d,d'}|<\epsilon$. Thus, for all $t\ge \max\{t_1,t_2\}$, 
\begin{equation}
\kappa_{d,d'}^{(t+1)}\ge \left(1+\eta_t\epsilon p(t)\right)\kappa_{d,d'}^{(t)}.
\label{eq:kt}
\end{equation}

Further,  from the conditions of the theorem, for almost all initializations, $|\hat{\u}^{(0)}_l[d]|>0$ for all $d$. For step sizes $\{\eta_t\}$ smaller than the local Lipschitz constant, for all finite $t^\prime<\infty$, we also have $|\u[t^\prime]_l[d]|>0$. Moreover from Lemma~\ref{lem:cn-lemma}, we have that $|\b{u}^{(t)}_d|,|\b{u}^{(t)}_{d'}|\to \infty$ and hence $\exists t_3$ such that $\forall t\ge t_3$, $|\b{u}^{(t)}_d|>0$, but for any finite $t'<\infty$, $|\b{u}^{(t')}_{d'}|<\infty$. Thus, for  $t_0=\max\{t_1,t_2,t_3\}$, using the above observations, we have that $\kappa_{d,d'}^{(t_0)}=\left|\frac{\b{u}^{(t_0)}_d}{\b{u}^{(t_0)}_{d'}}\right|>0$. 

Now, using eq. \eqref{eq:kt}, for all $t\ge t_0$, 
\begin{align}
\kappa_{d,d'}^{(t+1)}\ge 
(1+\eta_t\epsilon p(t))\kappa_{d,d'}^{(t)}=\left(\prod_{u=t_0}^t(1+\eta_u\epsilon p(u))\right)\kappa_{d,d'}^{(t_0)} \tand \kappa_{d,d'}^{(t_0)}>0.
\label{eq:prodk}
\end{align}

Finally, we show the following claim:
\begin{claim} For any finite $t_0$, finite step-sizes $\{\eta_t\}$, and any $\epsilon>0$, we have $\prod_{u=t_0}^t(1+\eta_u\epsilon p(u))\to\infty$.
\end{claim}
\begin{proof} Let $\mu=\max_{d} |\hat{\bar\z}^\infty[d]|+\max_{t>t_0}|\bdelta[t]_d|<\infty$. 
From eq. \eqref{eq:cn-ud++}, we have that for all $d$, \[|\b{u}^{(t+1)}_d|\le(1+\mu\eta_t p(t))|\b{u}^{(t)}_d|\le |\b{u}^{(t_0)}_d|\prod_{u=t_0}^t(1+\mu \eta_up(u))\le |\b{u}^{(t_0)}_d|\exp(\sum_{u=t_0}^t\mu\eta_u p(u)).\] 
Moreover, we have $\b{u}^{(t)}_d\to\infty$ for at least one $d$, and for any finite step sizes and  finite $t_0$, $|\b{u}^{(t_0)}_d|<\infty$. This then implies that for some  $\mu<\infty$, $\exp(\sum_{u=t_0}^t\mu\eta_u p(u))\to\infty\implies \sum_{u=t_0}^t \eta_up(u)\to\infty$. 
Thus, for any $\epsilon>0$, we also have $\prod_{u=t_0}^t(1+\epsilon \eta_up(u))\ge \epsilon\sum_{u=t_0}^t \eta_u p(u)\to\infty$.
\end{proof}

%Recall from eq. \eqref{eq:cn-ud} that $\kappa_{d,d'}^{(t)}=\left|\frac{\b{u}_d^{(t)}}{\b{u}_{d'}^{(t)}}\right|=\sqrt{\frac{|\w[t][d]|}{|\w[t][d']|}}$. 

From eq. \eqref{eq:prodk} and above claim, we conclude that for all $d,d'$, if $|\hat{\bar{\z}}^\infty[d]|>|\hat{\bar{\z}}^\infty[d']|$, then $\kappa_{d,d'}^{(t)}\to\infty$. 
\end{asparaenum}
This completes the proof of the theorem. \mybox
\subsubsection{Proof of Lemma~\ref{lem:cn-lemma}}\label{app:cnlem}
\lemcnlemma*
\begin{proof}

%We first note that, for almost all initializations, $|\hat\u^{(0)}_l[d]|>0$ for all $d$. Additionally,  for step sizes $\{\eta_t\}$ smaller than the local Lipschitz constant, for all finite $t^\prime<\infty$, we also have $|\u[t^\prime]_l[d]|>0$. 
Recalling $\boldsymbol{\tau}_d^{(t)}$ from eq. \eqref{eq:cn-ud+} and $\b{u}_d^{(t)}$ from eq. \eqref{eq:cn-ud}, we have the following:
\begin{equation}
\frac{\boldsymbol{\tau}_d^{(t)}}{\b{u}_d^{(t)}}=\ci\bdelta^{(t)}_{2,d}\frac{
{\hat\u}_2^{(t)*}[d]-{\hat\u}_1^{(t)}[d]\cdot\e^{-\ci\phi_{\hat{\bar\z}^\infty[d]}}}{\hat{\u}_1^{(t)}[d]\cdot\e^{-\ci\phi_{\hat{\bar\z}^\infty[d]}}+\hat{\u}_2^{(t)*}[d].}=\ci\bdelta^{(t)}_{2,d}\frac{1-\frac{|\hat\u_1^{(t)}[d]|}{|\hat\u_2^{(t)}[d]|}\cdot\e^{-\ci\phi_{\hat{\bar\z}^\infty[d]}+\ci\phi_{\hat{\w}^{(t)}[d]}}}{1+\frac{|\hat\u_1^{(t)}[d]|}{|\hat\u_2^{(t)}[d]|}\cdot\e^{-\ci\phi_{\hat{\bar\z}^\infty[d]}+\ci\phi_{\hat{\w}^{(t)}[d]}}}
\label{eq:tdud}
\end{equation}

For all $d$ if $\hat{\bar\w}^\infty[d]=\hat{\bar\u}_1^\infty[d]\cdot\hat{\bar\u}_2^\infty[d]\neq0$, the it is straightforward to see that $\frac{|\hat\u_1^{(t)}[d]|}{|\hat\u_2^{(t)}[d]|}=\frac{|\hat\u_1^{(t)}[d]|/g(t)}{|\hat\u_2^{(t)}[d]|/g(t)}\to\frac{|\hat{\bar\u}_1^{\infty}[d]|}{|\hat{\bar\u}_2^{\infty}[d]|}=1$ (from eq. \eqref{eq:cnt1}), and also that $\e^{-\ci\phi_{\hat{\bar\z}^\infty[d]}+\ci\phi_{\hat{\w}^{(t)}[d]}}\to\e^{-\ci\phi_{\hat{\bar\z}^\infty[d]}+\ci\phi_{\hat{\bar\w}^\infty[d]}}=1$ (from eq. \eqref{eq:cnt2}). 
This along with eq. \eqref{eq:tdud} gives us $\frac{\boldsymbol{\tau}_d^{(t)}}{\b{u}_d^{(t)}}\to0$.

 Moreover, since $|\hat\w^{(t)}[d]|\to\infty$, we have $|\hat\u_2^{(t)}[d]|$ or $|\hat\u_2^{(t)}[d]|\to\infty$. Further, using $\e^{-\ci\phi_{\hat{\bar\z}^\infty[d]}+\ci\phi_{\hat{\w}^{(t)}[d]}}\to1$, we  have $|\b{u}_d^{(t)}|=|\hat\u_2^{(t)}[d]|+|\hat\u_1^{(t)}[d]|\e^{-\ci\phi_{\hat{\bar\z}^\infty[d]}+\ci\phi_{\hat{\w}^{(t)}[d]}}\to\infty$.

We now only need to show that these results also hold for $d$ such that $\hat{\bar\w}^\infty[d]=0$. Recall from the assumptions of the theorem that even when $\hat{\bar\w}^\infty[d]=0$, $\exists \phi_{\hat{\bar\w}^\infty[d]}\in[0,2\pi)$ such that $\e^{\ci\phi_{\hat{\w}^{(t)}[d]}}\to\e^{\ci\phi_{\hat{\bar\w}^\infty[d]}}$. We now prove the lemma by showing the following steps for  $d$ such that $\hat{\bar\w}^\infty[d]=0$. :
\begin{compactenum}[Step 1.]
\item Show $\frac{|{\hat{\u}^{(t)}_1[d]}|}{|{\hat{\u}^{(t)}_2[d]}|}\to1$.
\item Show $\text{Re}(\e^{-\ci\phi_{\hat{\bar\z}^\infty[d]}+\ci\phi_{\hat{\bar\w}^{\infty}[d]}})=2\cos(\phi_{\hat{\bar\z}^\infty[d]}-\phi_{\hat{\bar\w}^{\infty}[d]})\ge 0$.
\end{compactenum}

\paragraph{Proof of lemma assuming Step 1 and Step 2 hold}

The above steps would imply that in eq. \eqref{eq:tdud},
\begin{compactitem}
\item the denominator satisfies 
\begin{equation}
\begin{split}
\left|1+\frac{|\hat\u_1^{(t)}[d]|}{|\hat\u_2^{(t)}[d]|}\cdot\e^{-\ci\phi_{\hat{\bar\z}^\infty[d]}+\ci\phi_{\hat{\w}^{(t)}[d]}}\right|\to&\left|1+\e^{-\ci\phi_{\hat{\bar\z}^\infty[d]}+\ci\phi_{\hat{\bar\w}^{\infty}[d]}}\right|\\
&\ge
\left|1+\text{Re}(\e^{-\ci\phi_{\hat{\bar\z}^\infty[d]}+\ci\phi_{\hat{\bar\w}^\infty[d]}})\right|\ge1.
\end{split}
\label{eq:cn-denom}
\end{equation}
\item the numerator satisfies
\begin{equation}
\left|\bdelta^{(t)}_{2,d}\left(1-\frac{|\hat\u_1^{(t)}[d]|}{|\hat\u_2^{(t)}[d]|}\cdot\e^{-\ci\phi_{\hat{\bar\z}^\infty[d]}+\ci\phi_{\hat{\w}^{(t)}[d]}}\right)\right|\le |\bdelta^{(t)}_{2,d}|\left|1+\frac{|\hat\u_1^{(t)}[d]|}{|\hat\u_2^{(t)}[d]|}\right|\to0.
\label{eq:cn-num}
\end{equation}
\end{compactitem}
These eqs. along with eq. \eqref{eq:tdud} in turn prove the lemma, \textit{i.e.},  $\frac{\boldsymbol{\tau}_d^{(t)}}{\b{u}_d^{(t)}}\to0$ and  $|\b{u}_d^{(t)}|\to\infty$. 

\paragraph{Showing Step 1 and Step 2}
\begin{asparaenum}[\textbf{Step} 1.]
\item \textit{Show $\frac{|{\hat{\u}^{(t)}_1[d]}|}{|{\hat{\u}^{(t)}_2[d]}|}\to1$.}

From the dynamics of $\hat{\u}^{(t)}_l[d]$ from eq. \eqref{eq:cn-grad-2}, we have the following,
\begin{equation}
\begin{split}
|\hat{\u}^{(t+1)}_1[d]|^2&=|\hat{\u}^{(t)}_1[d]|^2+\eta_t\hat{\z}^{(t)}[d]\cdot\hat{\w}^{(t)*}[d]+\eta_t\hat{\z}^{(t)*}[d]\cdot\hat{\w}^{(t)}[d]+\eta_t^2|\hat{\z}^{(t)}[d]|^2|\hat{\u}^{(t)}_2[d]|^2\\
|\hat{\u}^{(t+1)}_2[d]|^2&=|\hat{\u}^{(t)}_2[d]|^2+\eta_t\hat{\z}^{(t)}[d]\cdot\hat{\w}^{(t)*}[d]+\eta_t\hat{\z}^{(t)*}[d]\cdot\hat{\w}^{(t)}[d]+\eta_t^2|\hat{\z}^{(t)}[d]|^2|\hat{\u}^{(t)}_1[d]|^2
\end{split}
\label{eq:cn-ulnorm}
\end{equation}
Note that since $|\hat{\z}^{(t)}[d]|^2\to0$ and $\eta_t$ are finite, we have that $\exists t_1$ such that for all $t\ge t_1$, $\eta_t|\hat{\z}^{(t)}[d]|^2\le1$. From the above equation, we have the following for $t\ge t_1$,
\begin{equation}
\begin{split}
\left||\hat{\u}^{(t+1)}_1[d]|^2-|\hat{\u}^{(t+1)}_2[d]|^2\right|&=\left|\left(1-\eta_t^2|\hat{\z}^{(t)}[d]|^2\right)\left(|\hat{\u}^{(t)}_1[d]|^2-\hat{\u}^{(t)}_2[d]|^2\right)\right|\\
&\overset{(a)}{=}\left(\prod_{u=t_1}^t\left(1-\eta_u^2|\hat{\z}^{(u)}[d]|^2\right)\right)\left||\hat{\u}^{(t_1)}_1[d]|^2-\hat{\u}^{(t_1)}_2[d]|^2\right|\\
&\le \left||\hat{\u}^{(t_1)}_1[d]|^2-\hat{\u}^{(t_1)}_2[d]|^2\right|<\infty,
\end{split}
\end{equation}
where $(a)$ follows from iterating over $t$ and using $|\hat{\z}^{(t)}[d]|^2\le1$ for $t\ge t_1$. 

Since $|\hat{\w}^{(t)}[d]|=|\hat{\u}^{(t)}_1[d]|\cdot|\hat{\u}^{(t)}_2[d]|\to\infty$, at least one of $|\hat{\u}^{(t)}_1[d]|,|\hat{\u}^{(t)}_2[d]|$ must diverge. Without loss of generality, let $|\hat{\u}^{(t)}_2[d]|\to\infty$. Let $c(t):=|\hat{\u}^{(t)}_1[d]|^2-|\hat{\u}^{(t)}_2[d]|^2$ with $|c(t)|<\infty$. We  have 
\begin{equation}
\frac{|\hat{\u}^{(t)}_1[d]|^2}{|\hat{\u}^{(t)}_2[d]|^2}=1+\frac{c(t)}{|\hat{\u}^{(t)}_2[d]|^2}\overset{(a)}\to 1,
\end{equation}
where the convergence in $(a)$ follows since $|c(t)|<\infty$ (from eq. \eqref{eq:cn-ulnorm}) and $|\hat{\u}^{(t)}_2[d]|\to\infty$.

\item \textit{Show $\text{Re}(\e^{-\ci\phi_{\hat{\bar\z}^\infty[d]}+\ci\phi_{\hat{\bar\w}^{\infty}[d]}})=2\cos(\phi_{\hat{\bar\z}^\infty[d]}-\phi_{\hat{\bar\w}^{\infty}[d]})\ge 0$.} 

Note that from Step 1 above, we have that $\frac{|\hat{\u}^{(t)}_1[d]|^2}{|\hat{\u}^{(t)}_2[d]|^2}\to1$, which implies $\frac{|\hat{\u}^{(t)}_1[d]|^2+|\hat{\u}^{(t)}_2[d]|^2}{2|\hat{\w}^{(t)}[d]|}=\frac{|\hat{\u}^{(t)}_1[d]|^2+|\hat{\u}^{(t)}_2[d]|^2}{2|\hat{\u}^{(t)}_1[d]|\cdot|\hat{\u}^{(t)}_2[d]|}\to 1$. Thus, there exists $\bdelta[t]_{1,d}\to0$, such that 
\begin{equation}
|\hat{\u}^{(t)}_1[d]|^2+|\hat{\u}^{(t)}_2[d]|^2=2|\hat{\w}^{(t)}[d]|\cdot(1+\bdelta[t]_{1,d}).
\label{eq:cnt3}
\end{equation}
Also, from eq. \eqref{eq:cn-z}, there exists $\bdelta[t]_{2,d}\to0$, such that 
\begin{equation}
\hat{\z}^{(t)}[d]=\hat{\bar{\z}}^\infty[d] p(t)+\bdelta[t]_{2,d}p(t)\text{, with }p(t)=\norm{\hat{\z}^{(t)}}\to 0.
\label{eq:cnt4}
\end{equation}

Using the above representations, along with eq. \eqref{eq:cn-grad-2}, we have the following,
\begin{align}
\nonumber\hat{\w}^{(t+1)}[d]&=\hat{\w}^{(t)}[d]+\eta_t\hat{\z}^{(t)}[d]\left[|\hat{\u}^{(t)}_1[d]|^2+|\hat{\u}^{(t)}_2[d]|^2+\eta_t{\hat{\z}^{(t)}}[d]\cdot\hat{\w}^{(t)*}[d]\right]\\
\nonumber&\overset{(a)}=\hat{\w}^{(t)}[d]+2\eta_tp(t)|\hat{\w}^{(t)}[d]|\left(\hat{\bar{\z}}^\infty[d] +\bdelta[t]_{2,d}\right)\left[1+\bdelta[t]_{1,d}+\nicefrac{1}{2}\eta_t\hat{\z}^{(t)}[d]\e^{-\ci\phi_{\hat{\w}^{(t)}[d]}}\right]\\%\left[1+\bdelta[t]_{1,d}\right]\\
&\overset{(b)}{:=}\hat{\w}^{(t)}[d]+2\eta_tp(t)|\hat{\w}^{(t)}[d]|\left[\hat{\bar{\z}}^\infty[d] +\bdelta[t]_{3,d}\right],
\label{eq:cn-beta}
\end{align}
where $(a)$ follows from substituting eqs. \eqref{eq:cnt3}-\eqref{eq:cnt4}, and $(b)$ follows from using $|\hat{\z}^{(t)}[d]|\le p(t)\to0$ and defining $\bdelta[t]_{3,d}=\bdelta[t]_{2,d}\left[1+\bdelta[t]_{1,d}+\nicefrac{1}{2}\eta_t\hat{\z}^{(t)}[d]\e^{-\ci\phi_{\hat{\w}^{(t)}[d]}}\right]+\hat{\bar{\z}}^\infty[d]\left[\bdelta[t]_{1,d}+\nicefrac{1}{2}\eta_t\hat{\z}^{(t)}[d]\e^{-\ci\phi_{\hat{\w}^{(t)}[d]}}\right]\to0$. 

Denote $\Delta_d=\phi_{\hat{\bar\w}^{\infty}[d]}-\phi_{\hat{\bar{\z}}^\infty[d]}$. Additionally,  from the assumption in the theorem, we have $\e^{\ci\phi_{\hat{\w}^{(t)}[d]}}\to\e^{\ci\phi_{\hat{\bar\w}^{\infty}[d]}}$, hence there exists $\bdelta[t]_{4,d}\to0$ such that $\e^{\ci\phi_{\hat{\w}^{(t)}[d]}-\ci\phi_{\hat{\bar{\z}}^\infty[d]}}=\e^{\ci\Delta_d}(1+\bdelta[t]_{4,d})$. 

Now, from the above equation, for any $t_0$ and $t\ge t_0$, we derive the updates for $|\hat{\w}^{(t)}[d]|$,
\begin{align}
\nonumber|\hat{\w}^{(t+1)}[d]|^2&=|\hat{\w}^{(t)}[d]|^2\bigg(\e^{\ci\phi_{\hat{\w}^{(t)}[d]}}+2\eta_tp(t)\left[\hat{\bar{\z}}^\infty[d] +\bdelta[t]_{3,d}\right]\bigg)\bigg(\e^{-\ci\phi_{\hat{\w}^{(t)}[d]}}+2\eta_tp(t)\left[\hat{\bar{\z}}^{\infty*}[d] +\boldsymbol{\delta}^{(t)*}_{3,d}\right]\bigg)\\
\nonumber & \overset{(a)}=|\hat{\w}^{(t)}[d]|^2\left[1+2\eta_tp(t)\left(|\hat{\bar{\z}}^{\infty}[d]|\left(\e^{\ci\Delta_d}(1+\bdelta[t]_{4,d})+\e^{-\ci\Delta_d}(1+\bdelta^{(t)*}_{4,d})\right)+\bdelta[t]_{5,d}\right)\right]\\
\nonumber&\overset{(b)}{=}|\hat{\w}^{(t)}[d]|^2\left[1+4\eta_tp(t)\left(|\hat{\bar{\z}}^\infty[d]|\cos(\Delta_d) +\bdelta[t]_{6,d}\right)\right]\\
\nonumber&\overset{(c)}=|\hat{\w}^{(t_0)}[d]|^2\left[\prod_{u=t_0}^t \left(1+4\eta_up(u)\left(|\hat{\bar{\z}}^\infty[d]|\cos(\Delta_d) +\bdelta[u]_{6,d}\right)\right)\right]\\
&\overset{(d)}\le |\hat{\w}^{(t_0)}[d]|^2\exp(\sum_{u=t_0}^t4\eta_up(u)\left(|\hat{\bar{\z}}^\infty[d]|\cos(\Delta_d) +\bdelta[u]_{6,d}\right)),
\label{eq:cnbetamag}
\end{align}
where in $(a)$ we used $\e^{\ci\phi_{\hat{\w}^{(t)}[d]}-\ci\phi_{\hat{\bar{\z}}^\infty[d]}}=\e^{\ci\Delta_d}(1+\bdelta[t]_{4,d})$ and collected all $o(p(t))$ terms into $\bdelta[t]_{5,d}=\nicefrac{1}{2}\e^{\ci\phi_{\hat{\w}^{(t)}[d]}}\boldsymbol{\delta}^{(t)*}_{3,d}+2p(t)\hat{\bar{\z}}^\infty[d]\left[\hat{\bar{\z}}^{\infty*}[d] +\boldsymbol{\delta}^{(t)*}_{3,d}\right]+\bdelta[t]_{3,d}\bigg(\e^{-\ci\phi_{\hat{\w}^{(t)}[d]}}+2p(t)\left[\hat{\bar{\z}}^{\infty*}[d] +\boldsymbol{\delta}^{(t)*}_{3,d}\right]\bigg)\to0$ (since $p(t),\boldsymbol{\delta}^{(t)}_{3,d}\to0$); in $(b)$ we defined  $\bdelta[t]_{6,d}=\nicefrac{1}{2}\boldsymbol{\delta}^{(t)*}_{4,d}\e^{\ci\Delta_d}+\nicefrac{1}{2}\boldsymbol{\delta}^{(t)}_{3,d}\e^{-\ci\Delta_d}+\boldsymbol{\delta}^{(t)}_{5,d}\to0$; $(c)$ is obtained by iterating over $t$; and $(d)$ follows from using $(1+x)\le \exp(x)$.

If possible, let $\cos(\Delta_d)=-2\epsilon<0$. Since $|\bdelta[t]_{6,d}|\to0$, and for finite step sizes $\eta_tp(t)\to0$,  $\exists t_0$ such that for all $t\ge t_0$, $|\bdelta[t]_{6,d}|<\epsilon|\hat{\bar{\z}}^\infty[d]|$ and $\exp(-4\epsilon|\hat{\bar{\z}}^\infty[d]|\eta_tp(t))\le 1$. From eq.  \eqref{eq:cnbetamag}, we now have
\begin{align}
\nonumber|\hat{\w}^{(t+1)}[d]|^2\le |\hat{\w}^{(t_0)}[d]|^2\exp(-4\epsilon|\hat{\bar{\z}}^\infty[d]|\sum_{u=t_0}^t\eta_up(u))\le |\hat{\w}^{(t_0)}[d]|^2.
\label{eq:cnbetamag}
\end{align}
Finally, for any finite step sizes and finite $t_0$, we have $|\hat{\w}^{(t_0)}[d]|^2<\infty$ and this creates a contradiction since the LHS in the above equation diverges, $|\hat{\w}^{(t+1)}[d]|^2\to\infty$. Hence, in order for the updates in eq. \eqref{eq:cnbetamag} to lead to a divergent $|\hat{\w}^{(t+1)}[d]|$, we necessarily require that $\cos(\e^{\ci\Delta_d})=\text{Re}(\e^{-\ci\phi_{\hat{\bar\z}^\infty[d]}+\ci\phi_{\hat{\bar\w}^{\infty}[d]}})=2\cos(\phi_{\hat{\bar\z}^\infty[d]}-\phi_{\hat{\bar\w}^{\infty}[d]})\ge 0$.
\end{asparaenum}
This completes the proof of the lemma.
\end{proof}

\section{Computing $\c{R}_{\P}(\w)$: Proofs of Lemmas in Section~\ref{sec:Rw}}\label{app:lem}
In this appendix we prove the lemmas in Section~\ref{sec:Rw} that compute the form of induced bias of linear networks in the space of predictors. Recall that for linear predictors parameterized as $\w=\P(\u)$, $\c{R}_{\P}(\w)=\min_{\u:\P(\u)=\w}\norm{\u}_2^2$. 
\lemfcn*
\begin{proof} Recall that for fully connected networks of any depth $L>0$ with parameters  
${\u=[\u_l\in\bR^{D_{l-1}\times D_l}]_{l-1}^{L}}$, the equivalent  linear predictor given by $\P_{full}(\u)=\u_1\u_2\ldots\u_L$. 

We first show that $\c{R}_{\P_{full}}(\w)\ge L\norm{\w}_2^{\nicefrac{2}{L}}$. \\Let  $\u^\star({\w})=[\u^\star_{l}({\w})]_{l=1}^L$ be the minimizer of $\min_{\u:\P_{full}(\u)=\w} \norm{\u}_2^2$, so that $\w=\P_{full}(\u^\star({\w}))=\u^\star_1({\w})\cdot\u^\star_2({\w})\ldots\u^\star_L({\w})$ and $\c{R}_{\P_{full}}(\w)=\norm{\u^\star({\w})}_2^2=\sum_{l=1}^L\norm{\u^\star_l({\w})}_2^2$.  We then have, 
\begin{align}
\nonumber \norm{\w}_2^{\nicefrac{2}{L}}&=\norm{\u^\star_1({\w})\cdot\u^\star_2({\w})\ldots\u^\star_L({\w})}^{\nicefrac{2}{L}}_2\le \norm{\u^\star_1({\w})}_2^{\nicefrac{2}{L}}\norm{\u^\star_2({\w})}_2^{\nicefrac{2}{L}}\ldots\norm{\u^\star_L({\w})}_2^{\nicefrac{2}{L}}\\
&\overset{(a)}\le \frac{1}{L}\sum_{l=1}^L \norm{\u^\star_l(\w)}_2^2\! =\frac{1}{L}\c{R}_{\P_{full}}(\w),
\label{eq:rw-fcn}
\end{align}
where $(a)$ follows as arithmetic mean is greater than the geometric mean. 

Next, we show that $\c{R}_{\P_{full}}(\w)\le L\norm{\w}_2^{\nicefrac{2}{L}}$.\\
Given any unit norm vectors $\b{z}_l\in\bR^{D_l}$ for $l=1,2,\ldots,L$, consider $\bar{\u}=[\bar{\u}_l]$, defined as  
\[\bar{\u}_l=\left\{\begin{array}{ll}
\norm{\w}_2^{\nicefrac{1}{L}}\frac{\w}{\norm{\w}_2}\b{z}_1^\top&\text{if }l=1\\
\norm{\w}_2^{\nicefrac{1}{L}} \b{z}_{l-1}\b{z}_l^\top& \text{if }l=2,3,\ldots,L-1\\
\norm{\w}_2^{\nicefrac{1}{L}} \b{z}_{L-1}&\text{if }l=L
\end{array}\right.\]

% $l=2,3,\ldots,L-1$, $\bar{\w}_l=\norm{\w}_2^{\nicefrac{1}{L}} \b{z}_{l-1}\b{z}_l^\top$, $\bar{\b{U}}_L=\norm{\w}_2^{\nicefrac{1}{L}} \b{z}_{L-1}$, and $\bar{\b{U}}_1=\norm{\w}_2^{-1+{\nicefrac{1}{L}} }\w\b{z}_1^\top$. 

This ensures that $\P_{full}(\bar{\u})=\bar\u_1\bar\u_2\ldots\bar\u_L=\w$ and $\norm{\bar{\u}}_2^2=L\norm{\w}_2^{\nicefrac{2}{L}}$, and hence
\begin{equation}
\c{R}(\w)=\min_{\u:\P_{full}({\u})=\w}\norm{\u}_2^2\le \norm{\bar{\u}}_2^2=L\norm{\w}_2^{\nicefrac{2}{L}}. 
\label{eq:rw-fcn2}
\end{equation}

Combining  eq. \eqref{eq:rw-fcn} and eq. \eqref{eq:rw-fcn2}, we get $\c{R}_{\P_{full}}(\w)=L\norm{\w}_2^{2/L}$
\end{proof}

The proofs of the lemmas for computing $\c{R}_{\P}(\u)$ for diagonal and convolutional networks are similar to those of fully connected network.
\lemdn*
\begin{proof} Recall that for an $L$--layer linear diagonal networks with parameters 
${\u=[\u_l\in\bR^{D}]_{l-1}^{L}}$, the equivalent linear predictor is  given by $\P_{diag}(\u)=\text{diag}(\u_1)\text{diag}(\u_2)\ldots\text{diag}(\u_{L-1})\u_L$.

Let  $\u^\star(\w)=[\u^\star_l(\w)]_{l=1}^L$ be the minimizer of $\min_{\u:\P_{diag}(\u)=\w} \norm{\u}_2^2$, so that $\w=\P_{diag}(\u^\star(\w))$ and $\c{R}_{\P_{diag}}(\w)=\norm{\u^\star(\w)}_2^2$. We then have, %for all $d=0,1,\ldots,D-1$
\begin{align}
\nonumber\sum_{d=0}^{D-1} |\w{[d]}|^{\nicefrac{2}{L}}&=\sum_{d=0}^{D-1}\prod_{l=1}^L |\u_1^\star(\w)[d]|^{\nicefrac{2}{L}}\overset{(a)}\le \frac{1}{L} \sum_{d=0}^{D-1}\sum_{l=1}^L |\u_1^\star(\w)[d]|^2\\
&=\frac{1}{L}\norm{\u^\star(\w)}^2_2=\frac{1}{L} \c{R}_{\P_{diag}}(\w),
\label{eq:rw-dn}
\end{align}
where $(a)$ again follows as arithmetic mean is greater than the geometric mean. 

Similar to the case of fully connected networks, we now choose $\bar{\u}=[\bar{\u}_l]$ that satisfies $\P_{diag}(\bar{\u})=\w$ and $\norm{\bar{\u}}_2^2=L\norm{\w}_{\nicefrac{2}{L}}^{\nicefrac{2}{L}}$. This would ensure that, \[\c{R}_{\P_{diag}}(\w)=\min_{\u:\P_{diag}({\u})=\w}\norm{\u}_2^2 \le \norm{\bar{\u}}_2^2=L\norm{\w}_{\nicefrac{2}{L}}^{\nicefrac{2}{L}}.\]

We can check that these properties are satisfied by choosing $\bar{\u}$ as follows: for $d=0,1,\ldots D-1$, let  $\bar{\u}_1[d]=\text{sign}(\w[d])\,|\w[d]|^{\nicefrac{1}{L}}$ and $\bar{\u}_l[d]=|\w[d]|^{\nicefrac{1}{L}}$ for $l=2,3,\ldots,L$. 

Combining this argument with eq.~\ref{eq:rw-dn} concludes the proof. 
\end{proof}

For convolutional networks, the argument is the exactly the same as that for diagonal network adapted for complex vectors. 
\lemcn*
\begin{proof} 
Denote  the Fourier basis coefficients of  $\u_l\in\bR^{D}$  and $\w=\P_{conv}(\u)\in\bR^{D}$ in polar form as 
\[\hat\u_l=|\hat\u_l|\e^{\ci\boldsymbol{\phi}_{\hat\u_l}}\in\mathbb{C}^D,\quad \hat\w=|\hat\w|\e^{\ci\boldsymbol{\phi}_{\hat\w}}\in\mathbb{C}^D,\] where $|{\hat\u_l}|,|{\hat\w}|\in\bR^D_+$ and $\boldsymbol{\phi}_{\hat\u_l},\boldsymbol{\phi}_{\hat\w}\in[0,2\pi)^D$ are the vectors with magnitudes and phases, respectively, of $\hat\u_l,\hat\w$. 

From Lemma~\ref{lem:fft-conv}, the Fourier basis representation of  $\w=\P_{conv}(\u)$ is given by \[\hat{\w}=\text{diag}(\hat\u_1)\text{diag}(\hat\u_2)\ldots\text{diag}(\hat\u_{L-1})\hat\u_L=\P_{diag}(\hat{\u}),\] where we have overloaded the notation $\P_{diag}$ to denote the mapping of diagonal networks in complex vector fields, and $\hat{\u}=[\hat\u_l]_{l=1}^L$. We thus have for $d=0,1,\ldots,{D-1}$, 
\[|{\hat\w}[d]|=\prod_{l=1}^L|\hat\u_l[d]|,\quad\tand\quad\boldsymbol{\phi}_{\hat\w}[d]=\left(\sum_{l=1}^L\boldsymbol{\phi}_{\hat\u_l}[d]\right)\text{ mod }2\pi.\]

From orthonormality of discrete Fourier transformation, we have for all $\u$, $\norm{\u}_2^2=\norm{\hat{\u}}_2^2$. Thus,  

\begin{equation}
\c{R}_{\P_{conv}}(\w) =\min_{\u:\P_{conv}(\u)=\w} \norm{\u}_2^2=\min_{\hat\w:\hat{\w}=\P_\text{diag}(\hat{\u})}\norm{\hat\u}_2^2. 
\end{equation}

We can now  adapt the proof of diagonal networks here. 
Let  $\hat\u^\star(\w)=[\hat\u^\star_l(\w)\in\b{C}^D]_{l=1}^L$ be the minimizer of $\min_{\hat{\u}:\hat{\w}=\P_\text{diag}(\hat{\u})} \norm{\hat\u}_2^2$, so that $\hat{\w}=\P_{diag}(\hat{\u}^\star(\w))$  and $\c{R}_{\P_{conv}}(\w)=\norm{\hat{\u}^\star(\w)}_2^2$, and
\begin{align}
\nonumber\sum_{d=0}^{D-1} |\hat\w{[d]}|^{\nicefrac{2}{L}}&=\sum_{d}\prod_{l=1}^L |\hat\u_1^\star(\w)[d]|^{\nicefrac{2}{L}}\le \frac{1}{L} \sum_{d}\sum_{l=1}^L |\hat\u_1^\star(\w)[d]|^2\\
&=\frac{\norm{\hat{\u}^\star(\w)}_2^2}{L}=\frac{1}{L} \c{R}_{\P_{conv}}(\w).
\label{eq:rw-cn}
\end{align}
%where $(a)$ again follows as arithmetic mean is greater than the geometric mean. 

Similar to the diagonal networks,  we can choose the parameters in the Fourier domain   $\hat{\bar{\u}}=[\hat{\bar{\u}}_l\in\mathbb{C}^D]$ to ensure that $\P_{diag}(\hat{\bar{\u}})=\hat{\w}$  and $\norm{\hat{\bar{\u}}}_2^2=L\norm{\hat{\w}}_{\nicefrac{2}{L}}^{\nicefrac{2}{L}}$ as follows: for $d=0,1,\ldots D-1$, let  
\[\hat{\bar{\u}}_1[d]=\boldsymbol{\phi}_{\hat{\w}}[d]\,|\hat{\w}[d]|^{\nicefrac{1}{L}} \tand \hat{\bar{\u}}_l[d]=|\hat\w[d]|^{\nicefrac{1}{L}}, \forall l>1.\] 

This gives us \[\c{R}_{\P_{conv}}(\w)=\min\limits_{\u:\P_{diag}(\hat{\u})=\hat{\w}}\norm{\hat{\u}}_2^2\le\norm{\hat{\bar{\u}}}_2^2\le L\norm{\hat{\w}}_{\nicefrac{2}{L}}^{\nicefrac{2}{L}}.\]
%We can check that this property is satisfied by choosing $\hat{\bar{\u}}$ as follows: 
Combining this with eq.~\ref{eq:rw-cn} concludes the proof. 
\end{proof}
\section{Background Results}
\begin{theorem}[Stolz--Cesaro theorem, proof in Theorem~$1.22$ of \citet{muresan2009concrete}] Assume that $\{a_k\}_{k=1}^{\infty}$  and $\{b_k\}_{k=1}^{\infty}$ are two sequences of real numbers such that $\{b_k\}_{k=1}^{\infty}$ is strictly monotonic and diverging (i.e., monotone increasing with $b_k\to\infty$, or monotone decreasing with $b_k\to-\infty$).
Additionally, if $\lim_{k\to\infty} \frac{a_{k+1}-a_k}{b_{k+1}-b_k}=L$ exists, then $\lim_{k\to\infty} \frac{a_{k}}{b_{k}}$ exists and is equal to $L$. 
%\end{asparaenum}
\label{thm:stolzcesaro}
\end{theorem}

}
\end{document}